%% file: example_paper.tex
\documentclass[nohyperref]{article}

\usepackage{microtype}
\usepackage{graphicx}
\usepackage{subfigure}
\usepackage{booktabs} 

\usepackage{hyperref}

\usepackage[accepted]{icml2022}

\usepackage{amsmath}
\usepackage{amssymb}
\usepackage{mathtools}
\usepackage{amsthm}

\usepackage[capitalize,noabbrev]{cleveref}

\theoremstyle{plain}
\newtheorem{theorem}{Theorem}[section]
\newtheorem{proposition}[theorem]{Proposition}
\newtheorem{lemma}[theorem]{Lemma}
\newtheorem{corollary}[theorem]{Corollary}
\theoremstyle{definition}
\newtheorem{definition}[theorem]{Definition}
\newtheorem{assumption}[theorem]{Assumption}
\theoremstyle{remark}
\newtheorem{remark}[theorem]{Remark}

\usepackage[textsize=tiny]{todonotes}

\newcommand{\E}{\mathbb{E}}

\newcommand{\va}{{\mathbf {a}}}
\newcommand{\vg}{{\mathbf {g}}}
\newcommand{\vh}{{\mathbf {h}}}

\newcommand{\vs}{{\mathbf {s}}}

\newcommand{\vw}{{\mathbf {w}}}
\newcommand{\vx}{{\mathbf {x}}}
\newcommand{\vy}{{\mathbf {y}}}

\usepackage{enumitem}

\newcommand{\alglinelabel}{%
  \addtocounter{ALC@line}{-1}
  \refstepcounter{ALC@line}
  \label
}

\begin{document}

\onecolumn


\icmltitle{Kronecker-factored Quasi-Newton Methods for Deep Learning}



\icmlsetsymbol{equal}{*}

\begin{icmlauthorlist}
\icmlauthor{Yi Ren}{yyy}
\icmlauthor{Achraf Bahamou}{yyy}
\icmlauthor{Donald Goldfarb}{yyy}
\end{icmlauthorlist}

\icmlaffiliation{yyy}{Department of Industrial Engineering and Operations Research, Columbia University, New York NY, USA}

\icmlcorrespondingauthor{Yi Ren}{yr2322@columbia.edu}

\icmlkeywords{Machine Learning, ICML}

\vskip 0.3in





\printAffiliationsAndNotice{}  

\begin{abstract}

Second-order methods have the capability of accelerating optimization by using much richer curvature information than first-order methods. However, most are impractical for deep learning, where the number of training parameters is huge. In \citet{goldfarb2020practical}, practical quasi-Newton methods were proposed that approximate the Hessian of a multilayer perceptron (MLP) model by a layer-wise block diagonal matrix where each layer's block is further approximated by a Kronecker product corresponding to the structure of the Hessian restricted to that layer. Here, we extend these methods to enable them to be applied to convolutional neural networks (CNNs), by analyzing the Kronecker-factored structure of the Hessian matrix of convolutional layers. Several improvements to the methods in \citet{goldfarb2020practical} are also proposed that can be applied to both MLPs and CNNs. These new methods have memory requirements comparable to first-order methods and much less per-iteration time complexity than those in \citet{goldfarb2020practical}. Moreover, convergence results are proved for a variant under relatively mild conditions. Finally, we compared the performance of our new methods against several state-of-the-art (SOTA) methods on MLP autoencoder and CNN problems, and found that they outperformed the first-order SOTA methods and performed comparably to the second-order SOTA methods.

\end{abstract}



\section{Introduction}
\input{sections/intro}

\section{Kronecker-factored Structures in CNNs}
\label{sec_1}

\input{sections/conv}

\section{
Our New K-BFGS Method
}
\input{sections/QN}

\section{Space and Computational Requirements}
\input{sections/storage}

\section{Convergence Results}

\input{sections/proof}

\section{Numerical Results}
\label{sec_9}
\input{sections/experiment}

\section{Conclusion}
\input{sections/conclusion}

\clearpage

\bibliography{references.bib}
\bibliographystyle{icml2022}

\clearpage

\appendix
\onecolumn

\section{Pseudo-code and Implementation Details for K-BFGS / K-BFGS(L)}
\label{sec_14}
\input{sections/appendix/pseudo-code}

\section{Proof of Theorem \ref{thm_4}}
\label{sec_11}
\input{sections/appendix/proof_thm_4}

\section{Proof of Convergence for a Variant of
K-BFGS(L) and Associated Lemmas}
\label{sec_2}
\input{sections/appendix/convergence}

\section{Experiment Details}

\input{sections/appendix/experiment_appendix}

\end{document}

%% file: sections/intro.tex



    
    
    
    
    



First-order methods, including stochastic gradient descent (SGD) \citep{robbins1951stochastic} and the class of adaptive learning rate methods, such as AdaGrad \citep{duchi2011adaptive}, RMSprop \citep{hinton2012neural}, and Adam \citep{kingma2014adam}, are currently the most popular methods for training deep neural networks (DNNs),
including multilayer perceptrons (MLPs) and convolutional neural networks (CNNs), etc. Although these methods are fairly easy to implement, they use, at most, a very limited amount of curvature information to facilitate optimization. Vanilla SGD uses no curvature information, while the adpative learning rate methods use a diagonal pre-conditioning matrix based on the second moment of the gradient.

On the other hand, second-order methods use the rich curvature information of the problem to accelerate optimization.
Beside the classical Newton method, sub-sampled Newton methods have been proposed to handle large data sets (see e.g., \citet{xu2019newton}), but 
when the number of training parameters is huge, inverting the Hessian matrix is impractical. {Hence, q}uasi-Newton (QN) methods, ranging from the original BFGS \cite{broyden1970convergence,fletcher1970new,goldfarb1970family,shanno1970conditioning} and limited-memory BFGS (L-BFGS) \cite{liu1989limited}, to more recent developments that take into account stochasticity and/or non-convexity \citep{byrd2016stochastic,gower2016stochastic,wang2017stochastic},
{have been considered}. Other methods use surrogates to the Hessian, such as the Gauss-Newton (GN) and Fisher matrices (e.g., natural gradient (NG) method \cite{amari2000adaptive}, Hessian-free method \cite{martens2010deep}, Krylov subspace method \cite{vinyals2012krylov}, sub-sampled GN and NG methods \cite{ren2019efficient}, etc). However, in all of the above-mentioned second-order methods, whether they use the Hessian or a surrogate, the size of the curvature matrix becomes prohibitive when the number of training parameters is huge.

Consequently, second-order methods for training DNNs have been proposed that make use of layer-wise block-diagonal  approximations to Hessian and Fisher matrices, where each diagonal block is further approximated as the Kronecker product of smaller matrices 
 to reduce their memory and computational requirements.
One of the most popular methods of this type is the NG method KFAC, which was originally proposed for
MLPs \citep{martens2015optimizing}, and later extended to CNNs \citep{grosse2016kronecker} and other models \citep{wu2017scalable,martens2018kroneckerfactored}. 
Other Kronecker-factored NG methods have also been proposed in \citet{heskes2000,povey2014parallel,george2018fast}.
A Kronecker-factored QN method (which we will refer to as K-BFGS-20) was proposed in \citet{goldfarb2020practical}. This method was only designed for MLPs and serves as the starting point for the methods developed in this paper.
The approximate generalized GN method KFRA \cite{botev2017practical} adopts a block-diagonal Kronecker-factored approximation to 
the {GN matrix} for 
{MLPs} and computes the diagonal block approximations recursively. Finally,
Shampoo \citep{gupta2018shampoo} and TNT \citep{ren2021tensor} also use block-diagonal Kronecker-factored pre-conditioning matrices, stemming from adaptive learning rate methods and natural gradient methods, respectively

\subsection{Our Contributions}

In this paper, we propose brand new versions of K-BFGS, that are substantial extensions to the Kronecker-factored quasi-Newton methods for MLPs proposed in \citet{goldfarb2020practical}. Not only can they be applied to CNNs, but they also incorporate several generic improvements beyond  what was proposed in \citet{goldfarb2020practical}.


In order to enable K-BFGS to train CNN models, 
we first show that for a convolutional layer,
the gradient and Hessian restricted to that layer can be approximated by the Kronecker product of two vectors and matrices, respectively, extending the results in \citet{botev2017practical,wu2020dissecting}, etc. 
We then formalize exactly how K-BFGS should be applied to optimize the parameters in convolutional layers, {with the Kronecker-factored approximation of the Hessian as the basis}. 


{Our generic improvements to \citet{goldfarb2020practical} include} a new double damping technique $D_{P} D_{LM}$ and a "minibatched" Hessian-action BFGS,
both of which are applicable to convolutional and fully-connected layers. 



Our proposed methods
have comparable memory requirements to those of first-order methods, while their per-iteration time complexities are smaller, and in many cases, much smaller than those of {\citet{goldfarb2020practical} and other} popular second-order methods such as KFAC. Further, we prove convergence results for a limited memory K-BFGS(L) variant
under 
relatively mild conditions.


We conducted experiments on several MLP autoencoder problems, which demonstrated that our improved versions of K-BFGS outperformed the ones proposed in \citet{goldfarb2020practical}. Moreover, on several well-studied CNN problems, our proposed method outperformed the 1st-order SOTA methods SGD-m (i.e., SGD with momentum) and Adam and performed comparably to the 2nd-order SOTA method KFAC.


%% file: sections/conv.tex



{In this secton, after first discussing the computations involved in a CNN model, we describe the Kronecker structures of the gradient and Hessian of the loss function with respect to a convolutional layer's parameters.}



\subsection{Convolutional Neural Networks (CNNs)}

We consider a CNN with $L$ trainable layers (for simplicity, assume they are all convolutional layers),
with 
parameters consisting of
a weight tensor $w_l$ and a bias vector $b_l$ (shapes specified later) for $l \in \{ 1, ..., L \}$ and a loss function $\mathcal{L}$. For a data-point $(x, y)$, $x$ is fed into the CNN as input, yielding $\hat{y}$ as the output. 
The loss $\mathcal{L}(\hat{y}, y)$ between the output $\hat{y}$ and $y$ is a non-convex function of
{the set of all trainable parameters} 
$\theta := \{ w_1, b_1, ..., w_L, b_L \}$.

For a dataset that contains multiple data-points indexed by
{$n = 1, ..., N$}, let
{$f(n;\theta)$} denote the loss from the
{$n$th} data-point. Thus, viewing the dataset as an empirical distribution, the actual loss function that we wish to minimize is
{
$$f(\theta) := \E_n [f(n;\theta)] := \frac{1}{N} \sum_{n=1}^N f(n;\theta).$$
}

Let us now focus on a single convolutional layer
of the CNN, with its own weight tensor $w$ and bias vector $b$ as the trainable parameters.
For simplicity, we omit the layer index $l$, and assume that:
\begin{enumerate}[topsep=0pt,itemsep=-1ex,partopsep=1ex,parsep=1ex]
\item
the convolutional layer is 2-dimensional;
\item 
the filters are of size $(2R+1) \times (2R+1)$, with spatial offsets from the centers of each filter indexed by $\delta \in \Delta := \{ -R, ..., R \} \times \{ -R, ..., R \}$;
\item
the stride is of length 1, and the padding is equal to $R$, so that the sets of input and output spatial locations ($t \in \mathcal{T} \subset \mathbf{R}^2$) are the same.\footnote{The derivations in this paper can also be extended to the case where stride is greater than 1.};
\item 
the layer has $J$ input channels indexed by $j = 1, ..., J$, and $I$ output channels indexed by $i = 1, ..., I$.
\end{enumerate}


The
weight tensor $w \in \mathbf{R}^{I \times J \times (2R+1) \times (2R+1)}$ corresponds to the elements of all of the filters in the layer.
Hence, an element of $w$ is denoted as $w_{i,j,\delta}$, where the first two indices $i,j$ are the output/input channels, and the last two indices $\delta$ are the spatial offset within a filter. The bias $b \in \mathbf{R}^I$ is a length-$I$ vector.


Let $a$,
with components $a_{j, t}$,
denote the input to the layer after padding is added, where $t$ denotes the spatial location of the padded input and $j = 1, ..., J$; 
and let $h$, with components $h_{i, t}$, denote the 
output of the layer, where $t$ denote the spatial location of the output and $i = 1, ..., I$.
Given $a$, $h$ is computed as
\begin{align}
    h_{i,t} = \sum_{j=1}^J \sum_{\delta \in \Delta} w_{i,j,\delta} a_{j,t+\delta} + b_i,
    \quad t \in \mathcal{T}, \, i = 1, ..., I.
    \label{eq_12}
\end{align}
Note that we only consider the linear transformation of the convolutional layer. In other words,
if there is any activation or any other operations such as batch normalization afterwards, we view it as being separate from the layer.  

\subsection{Kronecker-factored Structure of Gradient and Hessian for Convolutional Layers}


Recent work has shown 
that curvature information in DNNs has the property of being a sum (or average, equivalently) of Kronecker products, beginning with the development of KFAC \citep{martens2015optimizing,grosse2016kronecker}, which showed this for Fisher matrices.
\citet{botev2017practical} and \citet{goldfarb2020practical} showed that Hessian matrices are also a sum of Kronecker product
for fully-connected layers, while \citet{bakker2018outer} and \citet{wu2020dissecting} extended this result to convolutional layers.

Based on the above Kronecker-factored structures, practical second-order methods for deep learning models were proposed, by approximating the curvature as a single Kronecker product, including KFAC \citep{martens2015optimizing,grosse2016kronecker}, KFRA \citep{botev2017practical}, and K-BFGS-20 \citep{goldfarb2020practical}.

Without developing any training method, \citet{wu2020dissecting} proposed a single Kronecker product approximation to the Hessian   
{of a DNN that consists of alternating} convolutional 
layers and ReLU activation functions, without any further modifications, such as batch normalization or skip connections
Moreover, assuming all activation functions are ReLU 
results in
the "second order" term in the Hessian being zero and the Hessian being equivalent to a Gauss-Newton matrix (or equivalently, a Fisher matrix). 
In contrast, we show that the Hessian 
has a Kronecker-factor structure for convolutional layers, without any assumptions about the activation functions or model architectures.

\subsubsection{Case 1: Single Data-point}




We now consider a single data-point, omitting the index $n$ for simplicity,
and derive the structure of gradient and Hessian with respect to the loss function $f(\cdot;\theta)$.

We define $\mathcal{D} X := \frac{\partial f}{\partial X}$ for any variable $X$ and ${\text{vec}}(\cdot)$ to be the vectorization of a matrix. For the
output and input of the layer, we define, respectively, the vectors
$$\vh_{t} := \left( h_{1,t}, ..., h_{I,t} \right)^\top \in \mathbb{R}^I,$$
\begin{align*}
    & \va_{t} := \left( a_{1,t+\delta_1}, ..., a_{J,t+\delta_{|\Delta|}}, 1 \right)^\top \in \mathbb{R}^{J|\Delta|+1},
\end{align*}
for $t \in \mathcal{T}$. Note that a homogeneous coordinate is concatenated at the end
of $\va_t$.

For the weights and biases, we define the vectors
\begin{align*}
    \vw_{i}
    & : = \left( w_{i,1,\delta_1}, ..., w_{i,J,\delta_{|\Delta|}}, b_i \right)^\top \in \mathbb{R}^{J|\Delta|+1},
\end{align*}
for $i = 1, ..., I$, 
and from them the matrix
\begin{align}
    {W} := ({\vw}_1, ..., {\vw}_I)^\top \in \mathbb{R}^{I \times (J|\Delta|+1)},
    \label{eq_19}
\end{align}
which contains all the parameters of the convolutional layer.

The following Theorem \ref{thm_4} gives the structure of gradient and Hessian of $W$ for a single data-point, the proof of which can be found in Section \ref{sec_11} in the Appendix. 

\begin{theorem}
For a single data-point, 

i) (Structure of gradient) $$\text{vec}( \mathcal{D} {W}) = \sum_{t \in \mathcal{T}} {\va}_t \otimes \mathcal{D} \vh_t$$ 
is the sum of $|\mathcal{T}|$ Kronecker products;

ii) (Structure of Hessian)
\begin{align*}
    \frac{\partial^2 f}{\partial \text{vec}(W)^2}
    = \sum_{t,t' \in \mathcal{T}} A_{t,t'} \otimes G_{t,t'}
\end{align*}
is the sum of $|\mathcal{T}|^2$ Kronecker products,
where 
\begin{align}
    & A_{t, t'}
    := \va_t \va_{t'}^\top \in \mathbb{R}^{(J|\Delta|+1) \times (J|\Delta|+1)},
    \label{eq_5}
    \\
    & G_{t,t'} := \frac{\partial^2 f}{\partial \vh_{t} \partial \vh_{t'}} \in \mathbb{R}^{I \times I}.
    \nonumber
\end{align}
for $t,t' \in \mathcal{T}$.

\label{thm_4}
\end{theorem}


If we {\bf assume} that $G_{t,t'} = 0$ for $t \neq t'$, we have that
\begin{align}
    \frac{\partial^2 f}{\partial \text{vec}(W)^2}
    \approx \sum_{t \in \mathcal{T}} A_{t, t} \otimes G_{t, t}.
    \label{eq_{1.5}}
\end{align}

As in other methods that have been proposed for training DNNs that use Kronecker factored approximations to Hessian or other pre-conditioning matrices \citep{martens2015optimizing,grosse2016kronecker,botev2017practical,goldfarb2020practical}, we further approximate $\frac{\partial^2 f}{\partial \text{vec}(W)^2}$ by a single Kronecker product.
To achieve this, we now approximate the average of the Kronecker products of a set of matrix pairs 
{$\{(U_t,V_t)\}$} by the Kronecker product of the averages of individual sets of matrices 
{$\{U_t\}$, $\{V_t\}$}, i.e.,
\begin{align}
\frac{1}{|\mathcal{T}|} \sum_{t \in \mathcal{T}} U_t \otimes V_t
\approx
\left( \frac{1}{|\mathcal{T}|} \sum_{t \in \mathcal{T}} U_t \right) \otimes \left( \frac{1}{|\mathcal{T}|} \sum_{t \in \mathcal{T}} V_t \right).
\label{eq_2}
\end{align}
Applying (\ref{eq_2}) to  (\ref{eq_{1.5}}), we have that
\begin{align}
    & \frac{\partial^2 f}{\partial \text{vec}(W)^2}
    \approx |\mathcal{T}| \cdot \left( \frac{1}{|\mathcal{T}|} \sum_{t \in \mathcal{T}} A_{t, t} \right) \otimes \left( \frac{1}{|\mathcal{T}|} \sum_{t \in \mathcal{T}} G_{t, t} \right)
    \nonumber
    \\
    = & \left( \sum_{t \in \mathcal{T}} A_{t, t} \right) \otimes \left( \frac{1}{|\mathcal{T}|} \sum_{t \in \mathcal{T}} G_{t, t} \right).
    \label{eq_2.5}
\end{align}
Note that the assumptions we made in deriving (\ref{eq_2.5}) are analogous to the IAD (Independent Activations and Derivatives), SH (Spatial Homogeneity), and SUD (Spatially Uncorrelated Derivatives) assumptions in \citet{grosse2016kronecker}. 
Lastly, one can similarly derives a single Kronecker approximation for the gradient, using Theorem \ref{thm_4} and (\ref{eq_2}).

\subsubsection{Case 2: Multiple Data-points}

In the case of multiple data-points, we use
$(n)$ to denote the index of a data-point. 
To approximate the average Hessian across multiple data-points as a single Kronecker product, 
we again use (\ref{eq_2}), but averaging over the data points this time. By (\ref{eq_2.5}),
we have that
%
%
%
%
\begin{align}
    & \frac{\partial^2 f}{\partial \text{vec}(W)^2}
    = \E_n \left[ \frac{\partial^2 f(n)}{\partial \text{vec}(W)^2} \right]
    \nonumber
    \\
    \approx \; & \E_n \left[ \left( \sum_{t \in \mathcal{T}} A_{t, t}(n) \right) \otimes \left( \frac{1}{|\mathcal{T}|} \sum_{t \in \mathcal{T}} G_{t, t}(n) \right) \right]
    \\
    \approx \; & A \otimes G,
    \label{eq_4}
\end{align}
where
\begin{align}
    A := \E_n \left[ \sum_{t \in \mathcal{T}} A_{t, t}(n) \right],
    G := \E_n \left[ \frac{1}{|\mathcal{T}|} \sum_{t \in \mathcal{T}} G_{t, t}(n) \right].
    \label{eq_8}
\end{align}
(\ref{eq_4}) will serve as the foundation for us to develop the Kronecker-factored QN method for CNNs below. 

%% file: sections/QN.tex
\subsection{K-BFGS-20 in \citet{goldfarb2020practical}}

For the Kronecker-factored quasi-Newton method
that \citet{goldfarb2020practical} proposed for training multiplayer perceptrons (MLPs), they approximated the Hessian of the loss function by a block diagonal matrix, where each block corresponds to the Hessian w.r.t. the parameters of a fully-connected layer. As a result, the parameters of each layer can be updated separately. 

For a single fully-connected layer in the MLP, \citet{goldfarb2020practical} approximates the Hessian restricted to that layer as $H_A \otimes H_G$, where $H_A$ and $H_G$ are some approximations to some matrices $A^{-1}$ and $G^{-1}$, respectively. As a results, by the property of Kronecker product,
they update the parameters of this fully-connected layer by computing
\begin{align}
    W^+
    = W - \alpha H_G (\mathcal{D} W) H_A,
    \label{eq_18}
\end{align}
where $W$ denotes the parameters (including weights and biases) in the fully-connected layer and $\alpha$ denotes the learning rate.

Furthermore, in \citet{goldfarb2020practical}, $H_A$ and $H_G$, as the approximations to $A^{-1}$ and $G^{-1}$, are estimated with the BFGS (or L-BFGS)
updating formula. To be more specific, given an approximation
{$H_G$} to the inverse of a symmetric matrix
$G$, the BFGS updating formula  
computes
\begin{align}
    H_G^+ = (I - \rho \vs_G \vy_G^\top) H (I - \rho \vy_G \vs_G^\top) + \rho \vs_G \vs_G^\top,
    \label{eq_13}
\end{align}
with given vectors
{$\vs_G$, $\vy_G$} 
{and $\rho = \frac{1}{\vy_G^\top \vs_G}$}.
$H_A$ is similarly computed with BFGS updating.

Lastly, the $(\vs_G, \vg_G)$ pairs used by $H_G$ is derived from the definition of $G$ as a Hessian matrix w.r.t. the output of the fully-connected layer. A double damping procedure with damping term $\lambda_G$ is proposed to deal with the non-convexity of $G$. For $H_A$, \citet{goldfarb2020practical} keeps track of an estimation to $A$ and generate the $(\vs_A, \vy_A)$ pairs with a "Hessian-action" approach, i.e. letting $\vs_A = A \vy_A + \lambda_A \vs_A$, where $\lambda_A$ is the damping term.

\subsection{What's New in Our K-BFGS Method?}

In this part, we describe the generic improvements of our new methods beyond K-BFGS-20, as well as how to extend the methods to convolutional layers. The complete pseudo-code and other implementation details are described in Sec \ref{sec_14} in the Appendix.

\subsubsection{Generic Improvements beyond K-BFGS-20}
\label{sec_13}

\input{sections/dd}

{\bf Improvement \#2: "minibatched" Hessian-action BFGS.} In approximating $A^{-1}$, \citet{goldfarb2020practical} uses the so-called "Hessian-action" approach, in which they
computes $A \vs_A$ with an estimation of $A$ from a moving average scheme with a given hyper-parameter on decaying. In other words, one needs to compute $A$ from each minibatch and always keep track of a moving average of it, which could be time consuming
when $A$ is large. {The large size of $A$ is particularly true for convolutional layers.}

In this paper, we propose a "minibatched" version of Hessian-action BFGS, i.e. computing $A \vs$ with $A$ estimated from only the current minibatch. By doing so, we avoid the explicit computation of $A$, replacing it with a direct matrix-vector product $A \vs$. Moreover, the hyper-parameter on decaying is no longer needed, which could potentially save effort in hyper-parameter tuning. This improvement is applicable to both fully-connected and convolutional layers, the latter of which is described in Section \ref{sec_12}.


\subsubsection{Extension to Convolutional Layers}
\label{sec_12}

For convolutional layers, we similarly use (\ref{eq_18}) to update the parameters $W$ defined in (\ref{eq_19}), where $H_G$ and $H_A$ corresponds to some approximations to the inverse of $G$ and $A$ defined in (\ref{eq_8}). 

To approximate the inverse of $G$ defined in (\ref{eq_8}), i.e. compting $H_G$, we use the BFGS updating formula (\ref{eq_13}), or L-BFGS. (We name our method K-BFGS and K-BFGS(L), respectively, when BFGS or L-BFGS is used for estimating $H_G$.) 
For a fixed data-point index $n$ and $t \in \mathcal{T}$, the $(\vs, \vy)$ pair for $G_{t,t}(n) = \frac{\partial^2 f(n)}{\partial \vh_{t}(n)^2}$ is $(\vs, \vy) = (\vh_t^+(n) - \vh_t(n), \mathcal{D} \vh_t^+(n) - \mathcal{D} \vh_t(n))$, where the "$+$" sign denotes that we compute the value after a step of the parameters has been taken. Hence, for a fixed $n$, the $(\vs, \vy)$ pair for $\frac{1}{|\mathcal{T}|} \sum_{t \in \mathcal{T}} G_{t, t}(n)$ is $\left( \overline{\vh^+}(n) - \overline{\vh}(n), \overline{\mathcal{D} \vh^+}(n) - \overline{\mathcal{D} \vh}(n) \right)$,
where $\overline{X}$ denotes the value of $X_t$ averaged over the spatial locations $\mathcal{T}$ for any quantity $X$, i.e., $\overline{X} := \frac{1}{|\mathcal{T}|} \sum_{t \in \mathcal{T}} X_t$. 
Finally, the $(\vs, \vy)$ for $G$ in (\ref{eq_8}) is
\begin{align}
    & \vs_G = \E_n \left[ \overline{\vh^+}(n) - \overline{\vh}(n) \right],
    \label{eq_6}
    \\
    & \vy_G = \E_n \left[ \overline{\mathcal{D} \vh^+}(n) - \overline{\mathcal{D} \vh}(n) \right].
    \label{eq_7}
\end{align}

Combining the above with the $D_{P} D_{LM}$ approach described in Section \ref{sec_13}, the final $(\vs, \vy)$ pair we use is $D_{P} D_{LM}(\vs_G, \vy_G)$. 

To approximate the inverse of $A$ defined in (\ref{eq_8}), i.e., computing $H_A$, we use the "minibatched" Hessian-action BFGS described in Section \ref{sec_13}. 
Given the current estimate $H_A$ of $A^{-1}$, the $(\vs, \vy)$ pair for updating $H_A$ are computed as: 
\begin{align}
    \vs_A = H_A \hat{\va}, \quad \vy_A = A \vs_A + \lambda_A \vs_A,
    \label{eq_1}
\end{align}
where
$\hat{\va} = \mathbb{E}_n [ \frac{1}{|\mathcal{T}|} \sum_{t \in \mathcal{T}} \va_t(n)] = \mathbb{E}_n \left[ \overline{\va(n)} \right]$
and $\lambda_A$ is the damping term.
Since $A$ is estimated from a minibatch, we compute $A s_A$ without explicitly computing $A$. 
    To be specific, by (\ref{eq_5}) and (\ref{eq_8}), 
    \begin{align}
        A \vs_A
        & = \E_n \left[ \sum_{t \in \mathcal{T}} \va_t(n) \va_t(n)^\top \right] \vs_A
        \nonumber
        \\
        & = \E_n \left[ \sum_{t \in \mathcal{T}} (\va_t(n)^\top \vs_A) \va_t(n) \right].
        \label{eq_14}
    \end{align}

Lastly, we propose to set the damping terms $\lambda_A = \sqrt{|\mathcal{T}|} \sqrt{\lambda}$, $\lambda_G = \frac{1}{\sqrt{|\mathcal{T}|}} \sqrt{\lambda}$ for a given overall damping hyper-parameter $\lambda$, which is shown to be better than setting $\lambda_A = \lambda_G = \sqrt{\lambda}$, which was proposed in \citet{goldfarb2020practical} for fully-connected layers. (See Section \ref{sec_18} in the Appendix for more discussion on this.)

%% file: sections/dd.tex
\begin{algorithm}[tb]
    \caption{$D_{P} D_{LM}$ 
    ($P$ stands for Powell's damping and $LM$ stands for Levenberg-Marquardt damping)
    }
    \label{algo_5}
    \begin{algorithmic}[1]
    
    \STATE
    {\bf Input:}
    $\vs$, $\vy$;
    {\bf Output:}
    $\tilde{\vs}$, $\tilde{\vy}$;
    {\bf Given:}
    $H$, $0 < \mu_1 < 1$, $\mu_2 > 0$
    
    \IF{$\vs^\top \vy < \mu_1 \vy^\top H \vy$}
    \STATE {$\theta_1 = \frac{(1-\mu_1) \vy^\top H \vy}{\vy^\top H \vy - \vs^\top \vy}$}
    \ELSE
    \STATE $\theta_1 = 1$
    \ENDIF
    

    \STATE 
    $\tilde{\vs} = \theta_1 \vs + (1-\theta_1) H \vy$
    \COMMENT{Powell's damping on $H$}

    
    \STATE 
    $\tilde{\vy} = \vy + \mu_2 \tilde{\vs}$
    \COMMENT{Levenberg-Marquardt damping on $H^{-1}$}

    \STATE
    {\bf return:} $\tilde{\vs}$, $\tilde{\vy}$
    
    \end{algorithmic}
\end{algorithm}

{\bf Improvement \#1: $D_{P} D_{LM}$.} 
In the double damping procedure proposed in \citet{goldfarb2020practical}, the parameter $\mu_2$ can only take values in $(0, 1]$, which restricts its interpretation as a Levenberg-Marquardt (LM) damping term. 

We propose a new procedure $D_{P} D_{LM}$ (Algorithm \ref{algo_5}), in which the parameter $\mu_2$ {($= \lambda_G$)} is more directly related to LM damping and can take any values in $(0, \infty)$. To see this connection, note that in Algorithm \ref{algo_5}, after Powell's damping on $H$, $\tilde{\vs}^\top \vy \ge \mu_1 \vy^\top H \vy \ge 0$. Hence, $\tilde{\vs}^\top \tilde{\vy} \ge \mu_2 ||\tilde{\vs}||^2$, which can be viewed as LM damping with a parameter of $\mu_2$, since $G$ is then lower bounded by $\mu_2 I$. 




%% file: sections/storage.tex

\begin{table*}[t]
  \caption{Storage Requirement}
  \vskip 0.15in
  \label{table_4}
  \centering
  \begin{tabular}{l|cccllllll}
    \hline              
    Algorithm
    & $\mathcal{D} W$ & $\mathcal{D} W \odot \mathcal{D} W$ & $A$ / $H_A$ & $G$ / $H_G$ & Total
    \\
    \hline
    K-BFGS
    & $O(I J |\Delta|)$
    & ---
    & $O(J^2 |\Delta|^2)$
    & $O(I^2)$
    & $O(J^2 |\Delta|^2 + I J |\Delta| + I^2)$
    \\
    K-BFGS(L)
    & $O(I J |\Delta|)$
    & ---
    & $O(J^2 |\Delta|^2)$
    & $O(p I)$
    & $O(J^2 |\Delta|^2 + I J |\Delta| + p I)$
    \\
    KFAC
    & $O(I J |\Delta|)$
    & ---
    & $O(J^2 |\Delta|^2)$
    & $O(I^2)$
    & $O(J^2 |\Delta|^2 + I J |\Delta| + I^2)$
    \\
    Adam
    & $O(I J |\Delta|)$
    & $O(I J |\Delta|)$
    & ---
    & ---
    & $O(I J |\Delta|)$
    \\
    \hline
  \end{tabular}
\end{table*}

\begin{table*}[t]
  \caption{Computation per iteration beyond that required for the minibatch stochastic gradient }
  \vskip 0.15in
  \label{table_5}
  \centering
  \begin{tabular}{l|cccllllll}
    \hline
    Algorithm
    & Additional pass
    & Curvature
    & Step $\Delta W_l$
    \\
    \hline
    K-BFGS
    & $O\left(\frac{m I J |\Delta| |\mathcal{T}|}{T}\right)$
    & $O\left(\frac{m J |\Delta| |\mathcal{T}| + J^2 |\Delta|^2 + m I |\mathcal{T}| + I^2}{T}\right)$
    & $O(I J^2 |\Delta|^2 + I^2 J |\Delta|)$
    \\
    K-BFGS(L)
    & $O\left(\frac{m I J |\Delta| |\mathcal{T}|}{T}\right)$
    & $O\left(\frac{m J |\Delta| |\mathcal{T}| + J^2 |\Delta|^2 + m I |\mathcal{T}| + p I}{T}\right)$
    & $O(I J^2 |\Delta|^2 + p I J |\Delta|)$
    \\
    KFAC
    &
    $O\left(\frac{m I J |\Delta| |\mathcal{T}|}{T_1}\right)$
    &
    $O\left(\frac{m (J^2 |\Delta|^2 + I^2) |\mathcal{T}|}{T_1} + \frac{J^3 |\Delta|^3 + I^3}{T_2}\right)$
    & $O(I J^2 |\Delta|^2 + I^2 J |\Delta|)$
    \\
    Adam
    & ---
    & $O(I J |\Delta|)$
    & $O(I J |\Delta|)$
    \\
    \hline
  \end{tabular}
\end{table*}

In this section, we compare the space and computational requirements of the proposed
K-BFGS and K-BFGS(L) methods with KFAC (see Algorithm \ref{algo_1} in the Appendix) and Adam, which are among the predominant 2nd- and 
1st-order methods, respectively, used to train CNNs. (One can also easily make similar comparison for MLPs.
)

We focus on one convolutional layer, with $J $ input channels, $I$ output channels, kernel size $|\Delta|$, and $|\mathcal{T}|$ spacial locations. Moreover, let $m$ denote the size of minibatches, $p$ denote the number of $(\vs, \vy)$ pairs for L-BFGS, $T$ denote the curvature update frequency for K-BFGS/K-BFGS(L), and
{$T_1$ and $T_2$} denote the frequency of statistics update and inverse update for KFAC, respectively.


From Table \ref{table_4}, one can see that
K-BFGS/K-BFGS(L) requires roughly the same amount of memory as KFAC.
Note that $I$ and $J$ are usually much larger than $|\Delta|$ in CNNs. For example, in VGG16 \cite{simonyan2014very}, $I$ and $J$ can be as large as 512 whereas $|\Delta| = 9$. Hence, as Table \ref{table_4} shows, the memory required by
K-BFGS/K-BFGS(L) is of the same order as that of Adam in terms of $I$ and $J$. 


In Table \ref{table_5}, besides the operations listed, each algorithm also needs to compute the minibatch gradient 
requiring $O(m I J |\Delta| |\mathcal{T}|)$ time. 
(Note that $|\mathcal{T}|$ is usually much larger than $I$ or $J$.)
{
Comparing with K-BFGS-20, K-BFGS improves time complexity due to the usage of "minibatched" Hessian-action BFGS. If K-BFGS-20 were used (with original Hessian-action BFGS), the first term $O\left(\frac{m J |\Delta| |\mathcal{T}|}{T}\right)$ of the "Curvature" column for K-BFGS and K-BFGS(L) would increase to $O\left(\frac{m J^2 |\Delta|^2 |\mathcal{T}|}{T}\right)$.
}
K-BFGS requires considerably less time to compute 
curvature information than KFAC. First,
{K-BFGS} avoids matrix inversion, whose complexity is $O(I^3)$ and $O(J^3 |\Delta|^3)$ (although this is amortized
in KFAC by $\frac{1}{T_2}$ by using the same inverse for $T_2$ iterations). Second,
{K-BFGS} avoids computing the $\Omega$ and $\Gamma$ matrices of KFAC from minibatch data, whose complexity is $m (J^2 |\Delta|^2 + I^2) |\mathcal{T}|$. 
Instead, we directly compute the $(\vs, \vy)$ pairs for $H_A$ and $H_G$, without explicitly forming $A$ or $G$.


\input{sections/fig_1}

%% file: sections/fig_1.tex
  

  

%% file: sections/proof.tex
In this section, we present convergence results for
a variant of
K-BFGS(L) (specifically, Algorithm \ref{algo_7} in the Appendix), following the framework in \citet{wang2017stochastic}.
For the purpose of simplicity, we assume that all layers are convolutional. (Our results also hold for MLPs or if the model contains both convolutional and fully-connected layers.)

There are several minor difference (described in Section \ref{sec_2} in the Appendix) between Algorithm \ref{algo_7} and our actual implementation of K-BFGS(L). In particular, $D_{P(I)} D_{LM}$ (see Algorithm \ref{algo_6} in the Appendix), rather than $D_{P} D_{LM}$, is used, which leads to the following lemma:


\begin{lemma}
\label{lemma_3}
The output of
{$D_{P(I)} D_{LM}$} satisfies: $\frac{\tilde{\vs}^\top \tilde{\vs}}{\tilde{\vs}^\top \tilde{\vy}} \le \frac{1}{\mu_2}$, $\frac{\tilde{\vy}^\top \tilde{\vy}}{\tilde{\vs}^\top \tilde{\vy}} \le \frac{1}{\mu_3}$, where
{$\mu_3 = \frac{\mu_1}{\mu_2 (1 + 2 \mu_1)}$}.
\end{lemma}

Consequently, one can prove the following two lemmas:

\begin{lemma}
\label{lemma_4}

{Suppose that we use $(\vs, \vy)$ for the BFGS update (\ref{eq_13}).}
If $\frac{\vs^\top \vs}{\vs^\top \vy} \le \frac{1}{\mu_2}$,
$\frac{\vy^\top \vy}{\vs^\top \vy} \le \frac{1}{\mu_3}$, then $ \| B^+\| \leq \| B \| + \frac{1}{ \mu_3} $ 
and $\|H^+\|  \leq (1 + \frac{1}{\sqrt{\mu_2 \mu_3}})^2 \|H\| + \frac{1}{\mu_2}$, where $B$ denotes the inverse of $H$. 
\end{lemma}

\begin{lemma}
\label{lemma_5}
In Algorithm \ref{algo_7}, for a given layer index $l = 1, ..., L$, there exist two positive constants $\underline{\kappa}_G^l$ and $\bar{\kappa}_G^l$, such that $\underline{\kappa}_G^l I \preceq H_G^l(k)  \preceq  \bar{\kappa}_G^l I$, $\forall k$. 
\end{lemma}


To apply  the convergence results in \citet{wang2017stochastic}
to Algorithm \ref{algo_7}, {we need to have that $H_A^l$, and hence that $H_l = H_A^l  \otimes H_G^l$, is bounded above and below by positive definite matrices, in addition to $H_G^l$. For this purpose and for satisfying other requirements needed to apply the theory in \citet{wang2017stochastic}, we make the following assumptions:}

\begin{assumption}
$f: \mathbb{R}^{d} \rightarrow \mathbb{R}$ is continuously differentiable. $f(\theta)$ is lower bounded by a real number $f^{\text {low }}$ for any $\theta \in \mathbb{R}^{d}$. $\nabla f$ is globally Lipschitz continuous with Lipschitz constant $L$, i.e., for any $\theta, \theta' \in \mathbb{R}^{d}$,
$\| \nabla f(\theta)-\nabla f(\theta') \| \leq L \| \theta - \theta' \|$. 
\label{assumption_1}
\end{assumption}

\begin{assumption}
For every iteration $k,$ we have
\begin{align*}
    & a) \ \mathbb{E}_{\xi_{k}}\left[g(\theta_{k}, \xi_{k})\right]=\nabla f(\theta_{k}),
    \\
    & b) \ \mathbb{E}_{\xi_{k}}\left[\left\|g(x_{k}, \xi_{k})-\nabla f(\theta_{k})\right\|^{2}\right] \leq \sigma^{2},
\end{align*}
where {$g$ is the minibatch gradient and} $\sigma>0$ is the noise level of the gradient estimation, and $\xi_{k}, k=1,2, \ldots$ are independent samples, and for a given $k$ the random variable $\xi_{k}$ is independent of $\left\{ \theta_{j} \right\}_{j=1}^{k}$.
\label{assumption_2}
\end{assumption}

\begin{assumption}
The inputs $a_{j,t}^l$'s to any layers are bounded, i.e. $\exists \varphi > 0$ s.t. $\forall l, j, t, |a_{j,t}^l| \le \varphi$. 
\label{assumption_4}
\end{assumption}

Note that AS. \ref{assumption_4} is relatively mild, in the sense that it is fulfilled if the activation functions of the model are all bounded (e.g. sigmoid, tanh, binary step), or some appropriate "normalization" is performed before the data are fed into each layer.

We now show that our block-diagonal approximation to Hessian is bounded below and above by positive definite matrices in Lemma \ref{lemma_1}, and  after that, applying Theorem 2.8 in \citet{wang2017stochastic} we obtain our main convergence result, Theorem \ref{thm_3}.
The complete proofs of all of the lemmas and the theorem in this section are deferred to Sec \ref{sec_4} in the Appendix.





\begin{lemma}
For Algorithm \ref{algo_7}, under the assumption AS. \ref{assumption_4}, 
(i)
$\hat{A}_l \preceq (J_l |\Delta| \varphi^2+1) |\mathcal{T}^l| I$, $\forall l$, and
\\
(ii) there exist two positive constants $\underline{\kappa}$ and $\bar{\kappa}$ , such that $\underline{\kappa} I  \preceq  H =  \text{diag} \{ H_1, ..., H_L \}  \preceq  \bar{\kappa} I$.
\label{lemma_1}
\end{lemma}





\begin{theorem}

Suppose that assumptions AS.\ref{assumption_1}, AS.\ref{assumption_2}, AS.\ref{assumption_4} hold for $\{ \theta_{k} \}$ generated by Algorithm \ref{algo_7}.
We also assume that $\alpha_{k}$ is specifically chosen as
$\alpha_{k}=\frac{\underline{\kappa}}{L \bar{\kappa}^{2}} k^{-\beta}$
with $\beta \in(0.5,1)$.
Then
\begin{align*}
    \frac{1}{K} \sum_{k=1}^{K} \mathbb{E}\left[\left\|\nabla f(\theta_{k})\right\|^{2}\right]
    \leq \frac{2 L\left(M_{f}-f^{l o w}\right) \bar{\kappa}^{2}}{\underline{\kappa}^{2}} K^{\beta-1}
    \\
    +\frac{\sigma^{2}}{(1-\beta) m} (K^{-\beta} - K^{-1}),
\end{align*}
where $K$ denotes the iteration number and $M_f$ is a positive constant. Moreover, for a given $\epsilon \in(0,1)$, to guarantee that $\frac{1}{K} \sum_{k=1}^{K} \mathbb{E}\left[\left\|\nabla f(\theta_{k})\right\|^{2}\right]<\epsilon,$ the number of iterations $K$ needed is at most $O\left(\epsilon^{-\frac{1}{1-\beta}}\right)$.

\label{thm_3}
\end{theorem}


Theorem \ref{thm_3} shows that Algorithm \ref{algo_7} converges to a stationary point for a (possibly) non-convex function $f$.
We note that under
{very similar assumptions,} Theorems 2.5 and 2.6 in \citet{wang2017stochastic} also hold for Algorithm \ref{algo_7}.

%% file: sections/experiment.tex
In this section, we describe two sets of experiments, namely, three MLP autoencoder problems and four CNN problems, comparing our proposed methods to other relevant methods (see Sec \ref{sec_10} for the detailed description of them) mentioned in our paper.

The results reported in the tables and plots are all based on runs using 5 different random seeds and the tuned best hyper-parameters (HPs) from a grid search specified below. The values reported in the tables and the solid curves depicted in the plots are derived from the averages of the 5 runs, while the shaded areas in the plots depict the $\pm \text{std}/\sqrt{5}$ range for the runs.
All experiments were run on a machine with 8 Xeon Gold 6248 CPUs with one Nvidia V100 GPU.

\subsection{Comparison with \citet{goldfarb2020practical}}
\label{sec_17}

\input{sections/table_1}

\input{sections/table_7}

Our first set of experiments are on three MLP autoencoder problems with
MNIST \citep{lecun1998gradient}, FACES, and CURVES \citep{hinton2006reducing} datasets, that have become standard for testing the  performance of algorithms to train DNNs, (e.g., see \citet{goldfarb2020practical}). 
We tested our proposed K-BFGS and K-BFGS(L) methods, their counterpart in \citet{goldfarb2020practical} (i.e., K-BFGS-20 and K-BFGS(L)-20), as well as three SOTA methods,
SGD with momentum (SGD-m), Adam, and KFAC. See Section \ref{sec_15} in the Appendix for the model {architecture} and dataset details. 

Minibatches of size 1000 were used for all three problems, as in \citet{goldfarb2020practical}. 
Each algorithm was run for a fixed amount of time for each problem (500 seconds for MNIST and CURVES, 2000 seconds for FACES). For all Kronecker-factored QN method, we set $T = 1$, 
except the one denoted as $\text{K-BFGS}^{\dagger}$, which used $T =20$. For KFAC, we set $T_1 = 1$ and $T_2 = 20$. 
These settings are exactly the same as those in \citet{goldfarb2020practical}.

As we are primarily interested in optimization performance in these experiments, we conducted a grid search on two hyper-parameters (HPs), namely, learning rate and damping, for all methods (only learning rate for SGD-m), and selected the best HP values that achieved the smallest loss on the training set. (See Section \ref{sec_15} in the Appendix for the searching range and best HP values selected.) 
Finally, we ran each algorithm with their best HPs, using 5 different random seeds, and reported the average loss in Table \ref{table_1}. (See Figures \ref{fig_3} to \ref{fig_5} in the Appendix for the training curves.)


From Table \ref{table_1}, we can clearly see that K-BFGS and K-BFGS(L) consistently outperformed their counterparts in \citet{goldfarb2020practical}, justifying the effectiveness of the generic improvements we incorporated in K-BFGS. (See Sec \ref{sec_16} in the Appendix for a more complete ablation study.)
Moreover, K-BFGS and K-BFGS(L) also performed better than the 1st-order methods in most cases except for Adam on CURVES. Lastly, our proposed methods performed similarly to KFAC, particularly when amortization was used (see $\text{K-BFGS}^{\dagger}$). Note that the KFAC baseline that we implemented was better than the one in \citet{goldfarb2020practical}, since it splits the overall damping term adaptively (see Line \ref{line_11}  of Algorithm \ref{algo_1} in the Appendix), rather than simply setting $\pi_l = 1$, which turned out to be an important factor for KFAC.




\subsection{CNNs: Generalization Performance}

\begin{figure}[t]
  \centering
    \includegraphics[width=0.49\textwidth,height=8cm]{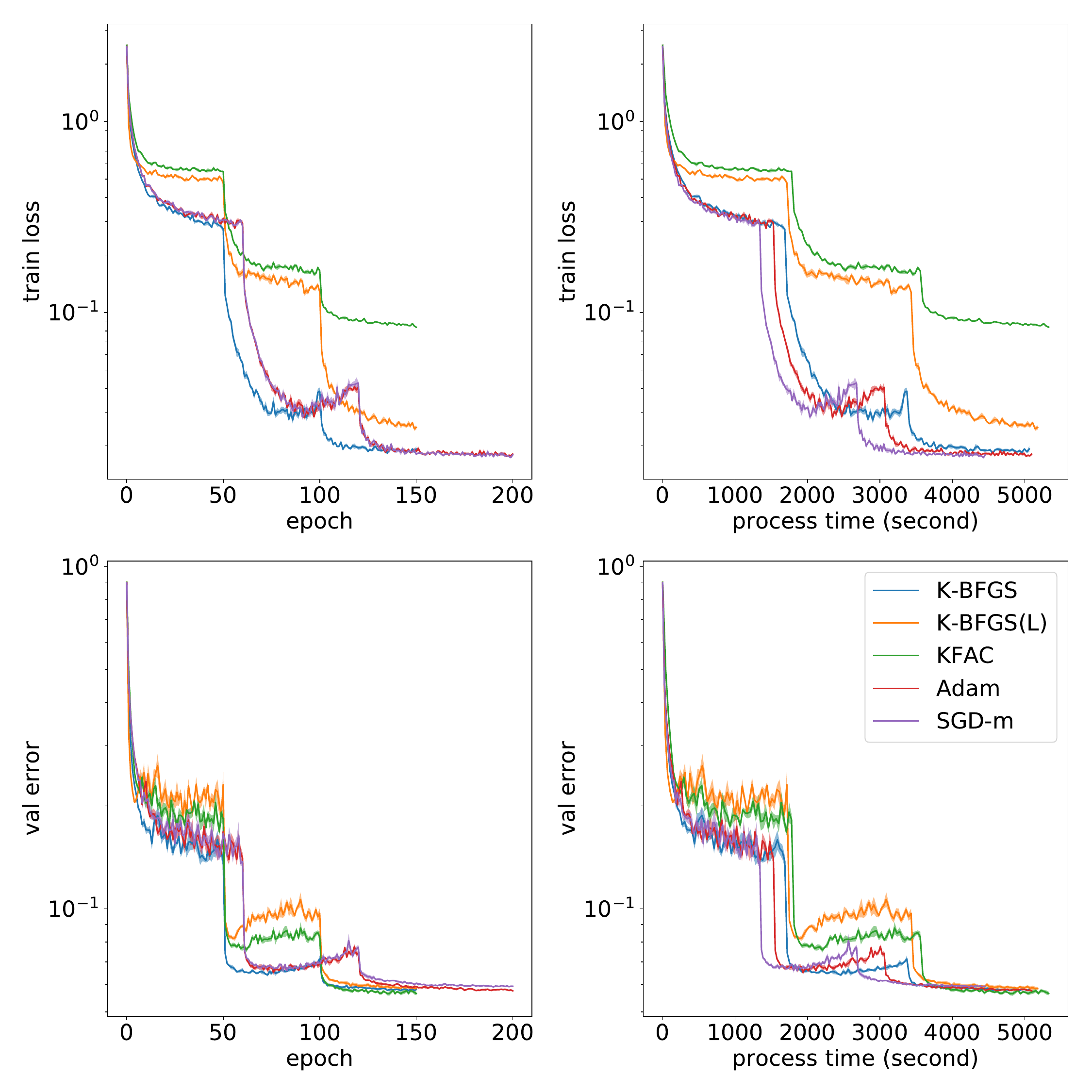}
  \caption{
   Training cross entropy loss (the upper row) and validation classification error (the lower row) against number of epochs (the left column) and process time (the right column) of K-BFGS, K-BFGS(L), KFAC, Adam, and SGD-m on VGG16 with CIFAR10. 
  }
  \label{fig_6}
\end{figure}

We tested K-BFGS,
K-BFGS(L),
KFAC, Adam, and SGD-m on two CNN models that have been found to be effective, namely,
VGG16 \cite{simonyan2014very} and ResNet32 \cite{he2016deep}.
We experimented on both models, using two datasets, CIFAR-10
and CIFAR-100
\cite{krizhevsky2009learning}, each of which includes 50,000 training samples and 10,000 testing samples (we view them as the validation set in our experiments). For both datasets, we applied the data augmentation techniques in \citet{NIPS2012_c399862d}, including random horizontal flip and random crop. (See Sec \ref{sec_5} in the Appendix for more details about the experimental set-up.) The above model/dataset choices have been used and endorsed in many papers, e.g. \citet{zhang2018three,choi2019empirical,ren2021tensor}.

Minibatches of size 128 were used for all experiments. 
For 
SGD-m and Adam, we ran the algorithms for 200 epochs and decay the learning rate by a factor of 0.1 every 60 epochs, which has been shown to be an effective learning rate schedule for these 1st-order methods on these problems. For 
K-BFGS, K-BFGS(L), and KFAC, we ran the algorithms for 150 epochs and decay the learning rate by a factor of 0.1 every 50 epochs, so that their overall running time is approximately the same as that of the 1st-order methods. For K-BFGS, we set the curvature update frequency $T = 20$. For KFAC, we set $T_1 = 10$ and $T_2 = 100$, as in \citet{zhang2018three}.

As we are interested in generalization performance in these experiments, we incorporated the weight decay technique (see Sec \ref{sec_5} in the Appendix) and  
conducted a grid search on three hyper-parameters (HPs), namely, initial learning rate, weight decay factor, and damping for all methods (only initial learning rate and weight decay for SGD-m). Then, we selected HP values that achieved the largest classification accuracy on the validation set (the grid search ranges and the best HP values so determined, are listed in Sec \ref{sec_5} in the Appendix), and reported the average classification accuracy on the validation sets in Table \ref{table_7}.
(See Figure \ref{fig_6} for the training and validation curves on VGG16+CIFAR10 and Figures \ref{fig_7}, \ref{fig_8}, and \ref{fig_9} in the Appendix for the others.)




Results in Table \ref{table_7} indicate that K-BFGS\footnote{Results in Table \ref{table_7} do not include K-BFGS(L) because it consistently underperformed K-BFGS. See Figure \ref{fig_6} and Figures \ref{fig_7}, \ref{fig_8}, and \ref{fig_9} in the Appendix for K-BFGS(L) results.} clearly outperformed Adam and SGD-m in terms of generalization, with the exception that K-BFGS and Adam achieved the same accuracy on VGG16+CIFAR10. Admittedly, KFAC achieved better accuracy than K-BFGS in 3 out of the 4 problems, but the gap is relatively small.




Finally, by comparing the process times reported in Figures 
\ref{fig_6} as well as Figures \ref{fig_7}, \ref{fig_8}, and \ref{fig_9} in the Appendix,
we can see that the per-iteration time of K-BFGS is only about 1/3 more than it is for the 1st-order methods. Moreover, 
the per-iteration time of K-BFGS (with $T=20$) is roughly the same as KFAC (with $T_2 = 100$), which demonstrates the effectiveness of our QN approach.

%% file: sections/table_1.tex
\begin{table}[b]
  \caption{
  Average of training loss achieved using 5 different random seed with best HP values. The dagger sign ($\dagger$) denotes that the curvature update frequency $T = 20$
  }
  \label{table_1}
  \centering
  \begin{tabular}{llll}
    \toprule
     & MNIST & FACES & CURVES \\
    \midrule
    K-BFGS & 51.60 & 5.00 & 55.46
    \\
    $\text{K-BFGS}^{\dagger}$ & 52.01 & 4.62 & 55.06
    \\
    K-BFGS(L)
    & 51.53 & 4.83 & 55.31
    \\
    K-BFGS-20 & 52.38 & 5.46 & 56.00
    \\
    K-BFGS(L)-20 & 54.27 & 4.92 & 55.94 
    \\
    KFAC  & 51.33 & 4.75 & 54.84
    \\
    Adam & 52.76 & 5.33 & 55.24
    \\
    SGD-m & 54.75 & 6.47 & 55.97
    \\
    \bottomrule
  \end{tabular}
\end{table}

%% file: sections/table_7.tex
\begin{table}[b]
  \caption{
  Average of validation classification accuracy (\%) achieved using 5 different random seeds with best HP values
  }
  \label{table_7}
  \centering
  \begin{tabular}{lllll}
    \toprule
    Dataset & \multicolumn{2}{c}{CIFAR10} & \multicolumn{2}{c}{CIFAR100}
    \\
    Model & VGG16 & ResNet32 & VGG16 & ResNet32
    \\
    \midrule
    K-BFGS & 94.28 & 93.45 & 75.65 & 71.46
    \\
    KFAC & 94.44 & 93.44 & 76.32 & 71.92
    \\
    Adam & 94.28 & 93.35  & 75.64  & 70.35
    \\
    SGD-m & 94.14 & 93.13 & 75.41 & 70.18
    \\
    \bottomrule
  \end{tabular}
\end{table}

%% file: sections/conclusion.tex

In this paper, we proposed a new class of Kronecker-factored quasi-Newton methods that are applicable to both MLP and CNN models, and that substantially improve upon the methods described in \citet{goldfarb2020practical}. We believe that our new methods are the first ones within the scope of quasi-Newton methods that use Kronecker-factored curvature approximations and  are practical for training CNNs.

With extensive numerical experiments, our new methods are shown to be better than the ones in \citet{goldfarb2020practical}. On several standard CNN models, K-BFGS outperforms SOTA first-order methods and performs similarly to KFAC.

%% file: sections/appendix/pseudo-code.tex

\subsection{Pseudo-code for K-BFGS / K-BFGS(L)}

\begin{algorithm}[h]
    \caption{Pseudo-code for K-BFGS / K-BFGS(L)
    }
    \label{algo_8}
    \begin{algorithmic}[1]
    
    
    
    \REQUIRE Given learning rates $\{ \alpha_k \}$, damping value $\lambda$, curvature update frequency $T$,
    batch size $m$
    
    \STATE $\mu_1 = 0.2$, $\beta = 0.9$
    
    \STATE $\lambda_A^l = \sqrt{|\mathcal{T}^l|} \sqrt{\lambda}$, $\lambda_G^l = \frac{1}{\sqrt{|\mathcal{T}^l|}} \sqrt{\lambda}$ ($l = 1, ..., L$) 
    \COMMENT{$\mathcal{T}^l$ denotes the sets of spatial locations in layer $l$}
    
    \STATE $\widehat{\mathcal{D} W_l} = 0$, $A_l = \mathbb{E}_n \left[ \sum_{t \in \mathcal{T}} \va_t^{l}(n) \va_t^{l}(n)^\top \right]$,
    {$H_A^l = (A_l + \lambda_A^l I_A)^{-1}$},
    {$H_G^l = (\lambda_G^l)^{-1} I$}, $\vs_G^l = \vy_G^l = 0$ ($l = 1, ..., L$)
    \COMMENT{Initialization}
    
    \FOR{$k = 1, 2, ...$}
    
    \STATE Sample mini-batch $M_k$ of size $m$
    \alglinelabel{line_12}

    \STATE Perform a forward-backward pass over $M_k$ to compute stochastic gradient $\widetilde{\mathcal{D} W_l}$ ($l = 1, ..., L$)
    
    \FOR{$l = 1, ..., L$}
    
        \STATE
        {$\widehat{\mathcal{D} W_l} = \beta \widehat{\mathcal{D} W_l} + \widetilde{\mathcal{D} W_l}$}
        
        \STATE $p_l = H_G^l \widehat{\mathcal{D} W_l} H_A^l$
        \COMMENT{if L-BFGS is used for $H_G^{l}$, it is initialized as
        $\lambda_G^{-1} I$}
        
        \STATE $W_l = W_l - \alpha_k p_l$
        \alglinelabel{line_14}
    \ENDFOR
    \alglinelabel{line_13}
    
    \IF{$k \equiv 0 \pmod{T}$}
    
    \STATE Perform another forward-backward pass over $M_k$ to compute
    {$\overline{\vh_l^+}$ and $\overline{\mathcal{D} \vh_l^+}$} ($l = 1, \ldots, L$)
    
    \FOR{$l = 1, ..., L$}
        
        \STATE \COMMENT{Update $H_A^l$ by BFGS}
    
    
    

        \STATE
        $\vs_A^l = H_A^l \widetilde{\overline{\va_l}}$,
        {$\vy_A^l = \widetilde{A_l} \vs_A^l + \lambda_A^l \vs_A^l$} using (\ref{eq_14})
        
        \STATE
        {
        Use BFGS updating (\ref{eq_13}) with $(\vs_A^l, \vy_A^l)$ to update $H_A^l$
        }

        \STATE \COMMENT{Update $H_G^l$ by
        BFGS or L-BFGS}
        
        \STATE 
        $\vs_G^l = \beta \vs_G^l + (1-\beta) \left( \widetilde{\overline{\vh^+_l}} - \widetilde{\overline{\vh_l}} \right), \vy_G^l = \beta \vy_G^l + (1-\beta) \left( \widetilde{\overline{\mathcal{D} \vh^+_l}} - \widetilde{\overline{\mathcal{D} \vh_l}} \right).$
        
        \STATE
        $(\tilde{\vs}_G^l, {\tilde{\vy}_G^l}) = D_{P}D_{LM}({\vs}_G^l, {{\vy}_G^l})$ with $H = H_G^l$, $\mu_1 = \mu_1$,
        {$\mu_2 = \lambda_G^l$}
        \COMMENT{See Algorithm \ref{algo_5}}  
            
        \STATE Use BFGS or L-BFGS with $(\tilde{\vs}_G^l, {\tilde{\vy}_G^l})$ to update $H_G^l$ \COMMENT{We name the algorithm K-BFGS or K-BFGS(L), respectively, when BFGS or L-BFGS is used.} 

    \ENDFOR
    
    \ENDIF
    
    \ENDFOR

    

    \end{algorithmic}
\end{algorithm}

Algorithm \ref{algo_8} gives the pseudo-code for
our proposed methods K-BFGS and K-BFGS(L).
Note that one can use either BFGS or L-BFGS update for $H_G$,
in which case we name the algorithm
K-BFGS and
K-BFGS(L), respectively. 
For simplicity, we assume that all layers in the model are convolutional layers. However, the algorithm can easily be adapted to fully-connected layers, hence applicable to MLP models or CNN models that contain fully-connected layers.


\subsection{Usage of Minibatches and Moving Averages}

Because there is usually a large amount of data, we use minibatches to estimate the quantities needed at every iteration.
We use $\widetilde{X}$ to denote the average value of $X$ over a minibatch for any quantity $X$, which is usually used as an estimate to $\E_n[X(n)]$.
Moreover, we use moving averages to both reduce the stochasticity and incorporate more information from the past:
\begin{itemize}[topsep=0pt,itemsep=-1ex,partopsep=1ex,parsep=1ex]
    \item
    {\bf Gradient.} At every iteration, the gradient $\widetilde{\mathcal{D} W}$ is estimated from a minibatch. We use a momentum scheme to get a better estimate $\widehat{\mathcal{D} W}$ of the gradient, i.e. we update
    $$\widehat{\mathcal{D} W} = \beta \widehat{\mathcal{D} W} + \widetilde{\mathcal{D} W}.$$

    \item {\bf BFGS updating for $H_G$.} By (\ref{eq_6}) and (\ref{eq_7}), we use both a minibatch and moving averages to estimate the $(\vs, \vy)$ for $H_G$, i.e. we update
    \begin{align*}
    & \vs_G = \beta \vs_G + (1-\beta) \left( \widetilde{\overline{\vh^+}} - \widetilde{\overline{\vh}} \right),
    \\
    & \vy_G = \beta \vy_G + (1-\beta) \left( \widetilde{\overline{\mathcal{D} \vh^+}} - \widetilde{\overline{\mathcal{D} \vh}} \right).
    \end{align*}

    \item {\bf BFGS updating for $H_A$.} In (\ref{eq_1}), we estimate the value of $A$ from the current minibatch, i.e. $\widetilde{\sum_{t \in \mathcal{T}} A_{t, t}}$, as well as $\hat{\va} = \widetilde{\overline{\va}}$. Note that $A \vs_A$ can be computed without forming $A$. 
\end{itemize}

\subsection{Other Details}
\label{sec_18}

Unlike \citet{goldfarb2020practical}, which always perform the whole K-BFGS process at every iteration, we introduce the so-called curvature update frequency $T$, controlling how frequently the algorithm update its curvature matrices. In other words, when $k \not\equiv 0 \pmod{T}$, only from Line \ref{line_12} to Line \ref{line_13}
of Algorithm \ref{algo_8} is incurred. 

Note that Algorithm \ref{algo_8} contains only one damping hyper-parameter (HP) $\lambda$, and sets $\lambda_A^l = \sqrt{|\mathcal{T}^l|} \sqrt{\lambda}$, $\lambda_G^l = \frac{1}{\sqrt{|\mathcal{T}^l|}} \sqrt{\lambda}$ for each convolutional layer $l = 1, ..., L$, where $\mathcal{T}^l$ denotes the sets of spatial locations in layer $l$. This can be viewed as "rebalancing" $A$ and $G$, i.e. setting $A$ to be $\E_n \left[ \frac{1}{\sqrt{|\mathcal{T}|}} \sum_{t \in \mathcal{T}} A_{t, t}(n) \right]$ and $G$ to be $\E_n \left[ \frac{1}{\sqrt{|\mathcal{T}|}} \sum_{t \in \mathcal{T}} G_{t, t}(n) \right]$. For fully-connected layers (if there is any), we set $\lambda_A^l = \lambda_G^l = \sqrt{\lambda}$, as in \citet{goldfarb2020practical}. Since $\lambda_A^l \lambda_G^l = \lambda$, adding $\lambda_A^l \vs_A^l$ and $\lambda_G^l {\vs}_G^l$, respectively, to the vectors $\vy_A^l$ and $\vy_G^l$ before applying BFGS (or L-BFGS) to  $H_A^l$ and  $H_G^l$, can be viewed as an approximation to adding the overall LM damping factor $\lambda I$  to $(H^l)^{-1} = (H_A^l)^{-1}  \otimes (H_G^l)^{-1} $ prior to updating.

{
Moreover, a warm start computation of $A_l$ is included, i.e. $A_l$ is computed from the whole dataset before first iteration, which is then used to initialize $H_A^l$. This only introduces a mild overhead, as this warm start computation take no more than the time for one full epoch, and gives a good starting point of $H_A^l$. 
}

When using the L-BFGS derived matrix $H_G^l$ to compute $H_G^l \widehat{\mathcal{D} W_l}$, instead of using the classical two-loop recursion of L-BFGS, we follow the "non-loop" implementation in \citet{byrd1994representations}, which is faster in practice because $\widehat{\mathcal{D} W_l}$ is a matrix, not a vector.

For CNN models that have batch normalization layers, we use the momentum gradient directions to update its parameters, with its own learning rate $\alpha_k / \lambda$. (Note that $\alpha_k / \lambda$ can be viewed as the "effective" learning rate in K-BFGS, roughly speaking.)

In terms of some default hyper-parameters, as shown in Algorithm \ref{algo_8}, decay parameters $\beta = 0.9$, and $\mu_1 = 0.2$ in $D_P D_{LM}$.
For K-BFGS(L),
the number of $(\vs, \vy)$ pairs stored for L-BFGS was set to be 100. These settings are exactly the same as in \citet{goldfarb2020practical}.

%% file: sections/appendix/proof_thm_4.tex
\begin{proof}

We first derive the structure of the gradient.
By (\ref{eq_12}), 
\begin{align*}
    \frac{\partial f}{\partial w_{i,j,\delta}}
    & = \sum_{t \in \mathcal{T}} \frac{\partial f}{\partial h_{i,t}} \frac{\partial h_{i,t}}{\partial w_{i,j,\delta}}
    = \sum_{t \in \mathcal{T}} \mathcal{D} h_{i,t} a_{j,t+\delta},
    \\ \frac{\partial f}{\partial b_{i}}
    & = \sum_{t \in \mathcal{T}} \frac{\partial f}{\partial h_{i,t}} \frac{\partial h_{i,t}}{\partial b_{i}}
    = \sum_{t \in \mathcal{T}} \mathcal{D} h_{i,t}.
\end{align*}
Hence,
\begin{align*}
    \mathcal{D} {\vw}_i = \sum_{t \in \mathcal{T}} \mathcal{D} h_{i,t} {\va}_t
    \Rightarrow
    \mathcal{D} {W} = \sum_{t \in \mathcal{T}} \mathcal{D} \vh_t \left( {\va}_t \right)^\top.
\end{align*}
and
$\text{vec}( \mathcal{D} {W}) = \sum_{t \in \mathcal{T}} {\va}_t \otimes \mathcal{D} \vh_t$.

{
To derive the Hessian of {$f(\cdot;\theta)$} for a single data-point, 
it follows from (\ref{eq_12}) that
\begin{align*}
    & \frac{\partial (\mathcal{D} h_{i,t})}{\partial w_{i',j,\delta}}
    = \sum_{t'} \frac{\partial (\mathcal{D} h_{i,t})}{\partial h_{i',t'}} \frac{\partial h_{i',t'}}{\partial w_{i',j,\delta}}
    = \sum_{t'} \frac{\partial^2 f}{\partial h_{i,t} \partial h_{i',t'}} a_{j,t'+\delta},
    \\
    & \frac{\partial (\mathcal{D} h_{i,t})}{\partial b_{i'}}
    = \sum_{t'} \frac{\partial (\mathcal{D} h_{i,t})}{\partial h_{i',t'}} \frac{\partial h_{i',t'}}{\partial b_{i'}}
    = \sum_{t'} \frac{\partial^2 f}{\partial h_{i,t} \partial h_{i',t'}}.
\end{align*}
Hence,
\begin{align*}
    \frac{\partial (\mathcal{D} h_{i,t})}{\partial \vw_{i'}}
    = \sum_{t'} \frac{\partial^2 f}{\partial h_{i,t} \partial h_{i',t'}} \va_{t'},
\end{align*}
and
\begin{align*}
    \frac{\partial^2 f}{\partial \vw_i \partial \vw_{i'}}
    = \frac{\partial}{\partial \vw_{i'}} \left( \sum_t \mathcal{D} h_{i,t} {\va}_t \right)
    = \sum_t \frac{\partial (\mathcal{D} h_{i,t})}{\partial \vw_{i'}} \va_t^\top
    = \sum_t \sum_{t'} \frac{\partial^2 f}{\partial h_{i,t} \partial h_{i',t'}} A_{t,t'}.
\end{align*}
Hence,
\begin{align*}
    \frac{\partial^2 f}{\partial \text{vec}(W)^2}
    = \sum_{t,t'} A_{t,t'} \otimes G_{t,t'}.
\end{align*}

}

\end{proof}

%% file: sections/appendix/convergence.tex
\begin{algorithm}[h]
    \caption{
    {K-BFGS(L)} with $D_{P(I)} D_{LM}$ and exact inversion of $A$
    }
    \label{algo_7}
    \begin{algorithmic}[1]
    
    
    
    \REQUIRE Given learning rates $\{ \alpha_k \}$, damping values $\lambda_A^l, \lambda_G^l > 0 \ (l=1,...,L)$,
    {batch size $m$}, {$0 < \mu_1 < 1$, $0 < \beta < 1$}
    
    
    \STATE $A_l = \mathbb{E}_n \left[ \sum_{t \in \mathcal{T}} \va_t^{l}(n) \va_t^{l}(n)^\top \right]$, $H_A^l = (A_l + \lambda_A^l I_A)^{-1}$,
    $H_G^l = (\lambda_G^l)^{-1} I$, $\vs_G^l = \vy_G^l = 0$ ($l = 1, ..., L$)
    \COMMENT{Initialization}
    
    \FOR{$k = 1, 2, ...$}
    
    \STATE Sample mini-batch $M_k$ of size $m$
    
    \STATE Perform a forward-backward pass over $M_k$ to compute stochastic gradient $\widetilde{\mathcal{D} W_l}$ ($l = 1, ..., L$)
    
    \FOR{$l = 1, ..., L$}
    
    
        

        
        \STATE $p_l = H_G^l \widetilde{\mathcal{D} W_l} H_A^l$
        \COMMENT{$H_G^{l}$ is initialized as
        {$\lambda_G^{-1} I$} in L-BFGS}
        
        \STATE $W_l = W_l - \alpha_k p_l$

    \ENDFOR
    
    \STATE Perform another forward-backward pass over $M_k$ to compute
    $\overline{\vh_l^+}$ and $\overline{\mathcal{D} \vh_l^+}$ ($l = 1, \ldots, L$)
    
    \FOR{$l = 1, ..., L$}
        \STATE \COMMENT{Compute $H_A^l$}
    
        \STATE Compute $\hat{A}_l = \widetilde{\sum_t A_{t, t}^l}$,
        $H_A^l = \left( \hat{A}_l + \lambda_A^l I \right)^{-1}$
        
        \STATE \COMMENT{Update $H_G^l$ by
        L-BFGS}
        
        \STATE 
        $\vs_G^l = \beta \vs_G^l + (1-\beta) \left( \widetilde{\overline{\vh^+_l}} - \widetilde{\overline{\vh_l}} \right), \vy_G^l = \beta \vy_G^l + (1-\beta) \left( \widetilde{\overline{\mathcal{D} \vh^+_l}} - \widetilde{\overline{\mathcal{D} \vh_l}} \right).$

        \STATE
        {$(\tilde{\vs}_G^l, {\tilde{\vy}_G^l}) = D_{P(I)}D_{LM}({\vs}_G^l, {{\vy}_G^l})$} with $H = H_G^l$, $\mu_1 = \mu_1$, $\mu_2 = \lambda_G^l$
        \COMMENT{See Algorithm \ref{algo_6}}  
        
            
        \STATE Use L-BFGS with $(\tilde{\vs}_G^l, {\tilde{\vy}_G^l})$ to update $H_G^l$
    
    \ENDFOR
    
    \ENDFOR

    

    \end{algorithmic}
\end{algorithm}

Algorithm \ref{algo_7} gives the variant of
{K-BFGS(L)} that is used in the convergence proof. Algorithm \ref{algo_7} differs from the actual implementation of K-BFGS(L), i.e. the one in Algorithm \ref{algo_8}, in the following:
\begin{enumerate}
    
    \item 
    $D_{P(I)} D_{LM}$ (Algorithm \ref{algo_6}) is used instead of $D_P D_{LM}$ (Algorithm \ref{algo_5}).
{$D_{P(I)} D_{LM}$} differs from $D_P D_{LM}$ by replacing $H$ by a scaled identity matrix $\mu_2^{-1} I$, where it appears in Algorithm \ref{algo_5}. This is justifiable partly because, in our actual implementation of L-BFGS, $H_G$ is always initialized with the scaled identity matrix $\mu_2^{-1} I$ where $\mu_2 = \lambda_G^l$;

    \item $H_A^l$ is computed by simply inverting $\hat{A}_l + \lambda_A^l I$, instead of using minibatched Hessian-action BFGS;
    
    \item 
    Gradient is estimated from the current minibatch without momentum.

\end{enumerate}

For simplicity, we also assume the curvature update frequency $T = 1$ in Algorithm \ref{algo_7}. However, all the proofs and results still hold if $T > 1$.

\begin{algorithm}[tb]
    \caption{$D_{P(I)}D_{LM}$}
    \label{algo_6}
    \begin{algorithmic}[1]
    
    \STATE
    {\bf Input:}
    $\vs$, $\vy$;
    {\bf Output:}
    $\tilde{\vs}$, $\tilde{\vy}$;
    {\bf Given:}
    $0 < \mu_1 < 1, \mu_2 > 0$

    \IF{
    {$\vs^\top \vy < \mu_1 \vy^\top (\mu_2^{-1} I) \vy$}
    }
    \STATE
    {$\theta_1 = \frac{(1-\mu_1) \vy^\top  \vy / \mu_2}{\vy^\top  \vy / \mu_2 - \vs^\top \vy}$}
    \ELSE
    \STATE $\theta_1 = 1$
    \ENDIF

    \STATE 
    {$\tilde{\vs} = \theta_1 \vs + (1-\theta_1) \mu_2^{-1} \vy$} 
    \COMMENT{Powell's damping with
    {$H = \mu_2^{-1} I$}}

    
    \STATE
    $\tilde{\vy} = \vy + \mu_2 \tilde{\vs}$
    \COMMENT{Levenberg-Marquardt damping on $H^{-1}$}

    \STATE
    {\bf return:} $\tilde{s}$, $\tilde{y}$
    
    \end{algorithmic}
\end{algorithm}

\subsection{Relevant Proofs for Theorem \ref{thm_3}
}
\label{sec_4}




\subsubsection{Proof of Lemma \ref{lemma_3}}
\begin{proof}

First, similar to Powell's damping on $H$, we can show that
{$\tilde{\vs}^\top \vy \ge \frac{\mu_1}{\mu_2} \vy^\top \vy \ge 0$}. Hence, 
$\tilde{\vs}^\top \tilde{\vy}  = \tilde{\vs}^\top \vy + \mu_2 \tilde{\vs}^\top \tilde{\vs} \ge \mu_2 \tilde{\vs}^\top \tilde{\vs}$. 

To see the second inequality, by using that $\tilde{\vs}^\top \tilde{\vy} \geq   \tilde{\vs}^\top \vy$
it follows that
{
\begin{align*}
    \tilde{\vy}^\top \tilde{\vy}
    & = \vy^\top \vy +2 \mu_2 \tilde{\vs}^\top \vy + 
    {\mu_2}^2  \tilde{\vs}^\top \tilde{\vs}
    = \vy^\top \vy +2 \mu_2 \tilde{\vs}^\top (\tilde{\vy} - \mu_2 \vs) + 
    {\mu_2}^2  \tilde{\vs}^\top \tilde{\vs}
    \\
    & \le \vy^\top \vy +2 \mu_2 \tilde{\vs}^\top \tilde{\vy}
    \leq \mu_2 (\frac{1}{\mu_1} + 2) \tilde{\vs}^\top \tilde{\vy}.
\end{align*}
}

\end{proof}

\subsubsection{Proof of Lemma \ref{lemma_4}}
\begin{proof}

Corresponding to the BFGS update (\ref{eq_13})  of $H$, the update of $B$ is
$$
B^+ = B - \frac{B \vs \vs^\top B}{\vs^\top B s} + \rho \vy \vy^\top.
$$
Hence,
\begin{align*}
    \| B^+\|
    & \leq \| B - \frac{B \vs \vs^\top B}{\vs^\top B \vs} \| + \|\rho \vy \vy^\top \|
    \leq \| B \| + \frac{\vy^\top \vy}{\vs^\top \vy}
    \leq \| B \| + \frac{1}{\mu_3}.
\end{align*}
Also, using the fact that for the spectral norm $\| \cdot \|$,
$\|  I - \rho \vs \vy^\top \| = \|  I - \rho \vy \vs^\top \|$, we have that


\begin{align*}
    \|H^+\| 
    \leq \|H\| \| I - \rho \vs \vy^\top \|^2 +  \|\frac{\vs \vs^\top}{\vs^\top \vy} \| 
    \leq\|H\| \Large( \|I\| +  \frac{\|\vs\| \| \vy\|}{\vs^\top \vy}\Large)^2
     +  \frac{\vs^\top \vs}{\vs^\top \vy} 
    \leq (1 + \frac{1}{\sqrt{\mu_2}}\frac{1}{\sqrt{\mu_3}})^2 \|H\| + \frac{1}{\mu_2}.
\end{align*}

\end{proof}

\subsubsection{Proof of Lemma \ref{lemma_5}}


\begin{proof}

To simplify notation, we omit the subscript $G$, superscript $l$ and the iteration index $k$ in the proof.
Hence, our goal is to prove $\underline{\kappa}_G I \preceq H = H_G^l(k) \preceq \bar{\kappa}_G I$, 
for any $k$.
Let
{$(\vs_i, \vy_i)$ ($i = 1, ..., p$)}
denote the pairs used in an L-BFGS computation of 
$H$. 


Given an initial estimate
{$H_0 =B_0^{-1}= \lambda_G^{-1} I$} of $(G_l(\theta_k))^{-1}$,
the L-BFGS method updates
$H_{i}$ recursively as
\begin{align}
    H_{i}
    =\left(I-\rho_{i} \vs_i \vy_i^{\top} \right) H_{i-1} \left(I-\rho_{i} \vy_{i} \vs_{i}^{\top}\right)+\rho_{i} \vs_{i} \vs_{i}^{\top}, 
\label{eq_23}
\end{align}
where $\rho_i = (\vs_i^\top \vy_i)^{-1}$, $i=1, ..., p$, and equivalently, 
\[
    B_{i}
    = B_{i-1} - \frac{B_{i-1} \vs_i \vs_i^{\top} B_{i-1}}{\vs_i^{\top} B_{i-1} \vs_i} + \rho_i {\vy}_i {\vy}_i^{\top}, \quad i=1, \ldots, p,
\]

where $B_i = H_i^{-1}$.
Since we use
{$D_{P(I)} D_{LM}$}, by Lemma \ref{lemma_3}, we have that 
$\frac{\vs_i^\top \vs_i}{\vs_i^\top \vy_i} \leq \frac{1}{\mu_2}$
and
$\frac{\vy_i^\top \vy_i}{\vs_i^\top \vy_i} \leq \frac{1}{\mu_3}$.

Hence, by Lemma \ref{lemma_4}, we have that
$||B_i|| \le ||B_{i-1}|| + \frac{1}{\mu_3}$. Hence,
{$||B|| = ||B_p|| \le ||B_0|| + \frac{p}{\mu_3} = \lambda_G + \frac{p}{\mu_3}$}. Thus,
{$B \preceq \left( \lambda_G + \frac{p}{\mu_3} \right) I$}, and
{$H \succeq \left( \lambda_G + \frac{p}{\mu_3} \right)^{-1} I \equiv \underline{\kappa}_G I$}. 

On the other hand, by Lemma \ref{lemma_4}, we have that $\|H_i\|  \leq (1 + \frac{1}{\sqrt{\mu_2 \mu_3}})^2 \|H_{i-1}\| + \frac{1}{\mu_2}$. Hence, from the fact that
{$H_0 = \lambda_G^{-1} I$}, and induction, we have that
{$||H|| \le \lambda_G^{-1} \hat{\mu}^p + \frac{\hat{\mu}^p - 1}{ \hat{\mu} - 1}\frac{1}{\mu_2} \equiv \bar{\kappa}_G$}, where $\hat{\mu} = (1 + \frac{1}{\sqrt{\mu_2 \mu_3}})^2$.

\end{proof}


\subsubsection{Proof of Lemma \ref{lemma_1}}


\begin{proof}
In proving part (i), we omit the layer index $l$ for simplicity. First, because $\hat{A}$ is the averaged value across the minibatch, it suffices to show that for any data-point $n$,
{$\sum_t A_{t, t}(n) \preceq (J |\Delta| \varphi^2+1) |\mathcal{T}| I$}.

By AS. \ref{assumption_4}, $||\va_{t}(n)||^2 = \sum_{j, \delta} a_{j, t+\delta}(n)^2 + 1 \le J |\Delta| \varphi^2 + 1$. Hence, for any vector $\vx$, 
\begin{align*}
    \vx^\top A_{t, t}(n) \vx
    = \vx^\top (\va_t(n) \va_{t}(n)^\top) \vx
    = (\va_{t}(n)^\top \vx)^2
    \le ||\va_{t}(n)||^2 ||\vx||^2
    \le (J |\Delta| \varphi^2+1) ||\vx||^2.
\end{align*}
Hence, $A_{t, t}(n) \preceq (J |\Delta| \varphi^2+1) I$ and $\sum_t A_{t, t}(n) \preceq (J |\Delta| \varphi^2+1) |\mathcal{T}| I$,  proving part (i). 




Note that $\hat{A}_l + \lambda_A^l I \succeq \lambda_A^l I$ because $\hat{A}_l$ is PSD. 
On the other hand, $\hat{A}_l + \lambda_A^l I \preceq ((J_l |\Delta| \varphi^2 + 1) |\mathcal{T}^l| + \lambda_A^l) I$. 
Hence, 
\begin{align}
    ((J_l |\Delta| \varphi^2 + 1) |\mathcal{T}^l| + \lambda_A^l)^{-1} I \preceq H_A^l 
    = (\hat{A}_l + \lambda_A^l I)^{-1} \preceq (\lambda_A^l)^{-1} I.
    \label{eq_9}
\end{align}

By (\ref{eq_9}) and Lemma \ref{lemma_5}, we have that $\underline{\kappa}^l I \preceq H_A^l \otimes H_G^l = H_l \preceq \overline{\kappa}^l I$ where $\underline{\kappa}^l = ((J_l |\Delta| \varphi^2 + 1) |\mathcal{T}^l| + \lambda_A^l)^{-1} \underline{\kappa}_G^l$, $\overline{\kappa}^l = (\lambda_A^l)^{-1} \overline{\kappa}_G^l$. Finally, $\underline{\kappa} I \preceq H = \text{diag} \{ H_1, ..., H_L \} \preceq \overline{\kappa} I$, where $\underline{\kappa} = \min \{ \underline{\kappa}^1, ..., \underline{\kappa}^L \}$ and $\overline{\kappa} = \max \{ \overline{\kappa}^1, ..., \overline{\kappa}^L \}$. 

\end{proof}


\subsubsection{Proof of Theorem \ref{thm_3}}

\begin{proof}
First, Algorithm \ref{algo_7} falls in the general framework of the Stochastic Quasi-Newton (SQN) method (Algorithm 2.1) in \citet{wang2017stochastic}. 
Second, by Lemma \ref{lemma_1}, Assumption AS.3 in \citet{wang2017stochastic} is satisfied. Also, by the way $H_A$ and $H_G$ are updated, AS.4 in \citet{wang2017stochastic} is satisfied. Hence, since Assumptions AS.1 and AS.2 are identical to the other two assumptions made in \citet{wang2017stochastic}, we are able to apply Theorem 2.8 in that paper to Algorithm 3 in this Section.. 

\end{proof}





%% file: sections/appendix/experiment_appendix.tex
\input{sections/appendix/implementation}


\subsection{Details on the Autoencoder Experiments}
\label{sec_15}

The autoencoder architectures are exactly the same as in \citet{goldfarb2020practical}.
The only difference is that we didn't include the regularization term $\frac{\eta}{2} ||\theta||^2$, because we focus on optimization performance so it is better to avoid this compounding factor and it is hard to know how to set the value of $\eta$ unless we include it in the HP tuning process. To be more specific, Table \ref{table_3} describes the model architectures. The activation functions of the hidden layers are always ReLU, except that there is no activation for the very middle layer. 

\begin{table}[tb]
  \caption{
  Model architectures for the MLP autoencoder problems
  }
  \label{table_3}
  \vskip 0.15in
  \centering
  \begin{tabular}{cccccccccc}
    \toprule
    & Layer width
    & Loss function
    \\
    \midrule
    MNIST
    &  [784, 1000, 500, 250, 30, 250, 500, 1000, 784]
    & binary cross entropy with sigmoid
    \\
    FACES
    & [625, 2000, 1000, 500, 30, 500, 1000, 2000, 625]
    & mean squared error
    \\
    CURVES
    &  [784, 400, 200, 100, 50, 25, 6, 25, 50, 100, 200, 400, 784]
    & binary cross entropy with sigmoid
    \\
    \bottomrule
  \end{tabular}
\end{table}

As in \citet{goldfarb2020practical}, we only use the training sets of the datasets, which contains 60k (MNIST\footnote{\url{http://yann.lecun.com/exdb/mnist/}}), 103.5k (FACES\footnote{\url{http://www.cs.toronto.edu/~jmartens/newfaces_rot_single.mat}}), and 20k (CURVES\footnote{\url{http://www.cs.toronto.edu/~jmartens/digs3pts_1.mat}}) training samples, respectively.


In order to obtain the results in Table \ref{table_1}, we first conducted a grid search for each algorithm based on the following ranges:
\begin{itemize}

\item learning rate: $\{$ 3e-5, 1e-4, 3e-4, 1e-3, 3e-3, 0.01, 0.03, 0.1, 0.3, 1 $\}$;

\item damping:
\begin{itemize}
    \item Kronecker-factored QN methods (i.e., $\lambda$ in Algorithm \ref{algo_8}) and KFAC (i.e., $\lambda$ in Algorithm \ref{algo_1}): $\{$ 0.03, 0.1, 0.3, 1, 3, 10, 30 $\}$;
    
    \item Adam (i.e., the $\epsilon$ HP in \citet{kingma2014adam}): $\{$ 1e-8, 1e-4, 1e-3, 0.01, 0.1 $\}$.
\end{itemize}
    
\end{itemize}

We selected the HP values that achieves the smallest training loss (see Table \ref{table_2}). We then ran each algorithm with their corresponding best HP values and 5 different random seed, and reported the average loss in Table \ref{table_1}. The training curves are also included in Figures \ref{fig_3}, \ref{fig_4}, and \ref{fig_5}, where the training loss is reported against number of epochs (left) and process time (right).

\begin{table}[tb]
  \caption{
  Best HP values (learning rate, damping) for Table \ref{table_1} as well as Figures \ref{fig_3}, \ref{fig_4}, and \ref{fig_5}.
  The dagger sign ($\dagger$) denotes that the curvature update frequency $T = 20$
  }
  \label{table_2}
  \vskip 0.15in
  \centering
  \begin{tabular}{cccccccccc}
    \toprule
    & K-BFGS
    & $\text{K-BFGS}^{\dagger}$
    & K-BFGS(L)
    & K-BFGS-20
    & K-BFGS(L)-20
    & KFAC
    & Adam
    & SGD-m
    \\
    \midrule
    MNIST
    & (0.03, 0.3)
    & (0.3, 30)
    & (0.03, 0.3)
    & (0.01, 0.1)
    & (0.003, 0.1)
    & (0.3, 10)
    & (1e-4, 1e-4)
    & (0.003, -)
    \\
    FACES
    & (0.03, 1)
    & (0.03, 3)
    & (0.03, 1)
    & (0.01, 0.3)
    & (0.01, 0.3)
    & (0.03, 0.1)
    & (1e-4, 1e-4)
    & (0.001, -)
    \\
    CURVES
    & (0.03, 0.3)
    & (0.3, 30)
    & (0.1, 3)
    & (0.01, 0.03)
    & (0.003, 0.03) 
    & (0.3, 10)
    & (1e-3, 1e-3)
    & (0.003, -)
    \\
    \bottomrule
  \end{tabular}
\end{table}

\begin{figure}[H]
  \centering
    \includegraphics[width=0.6\textwidth]{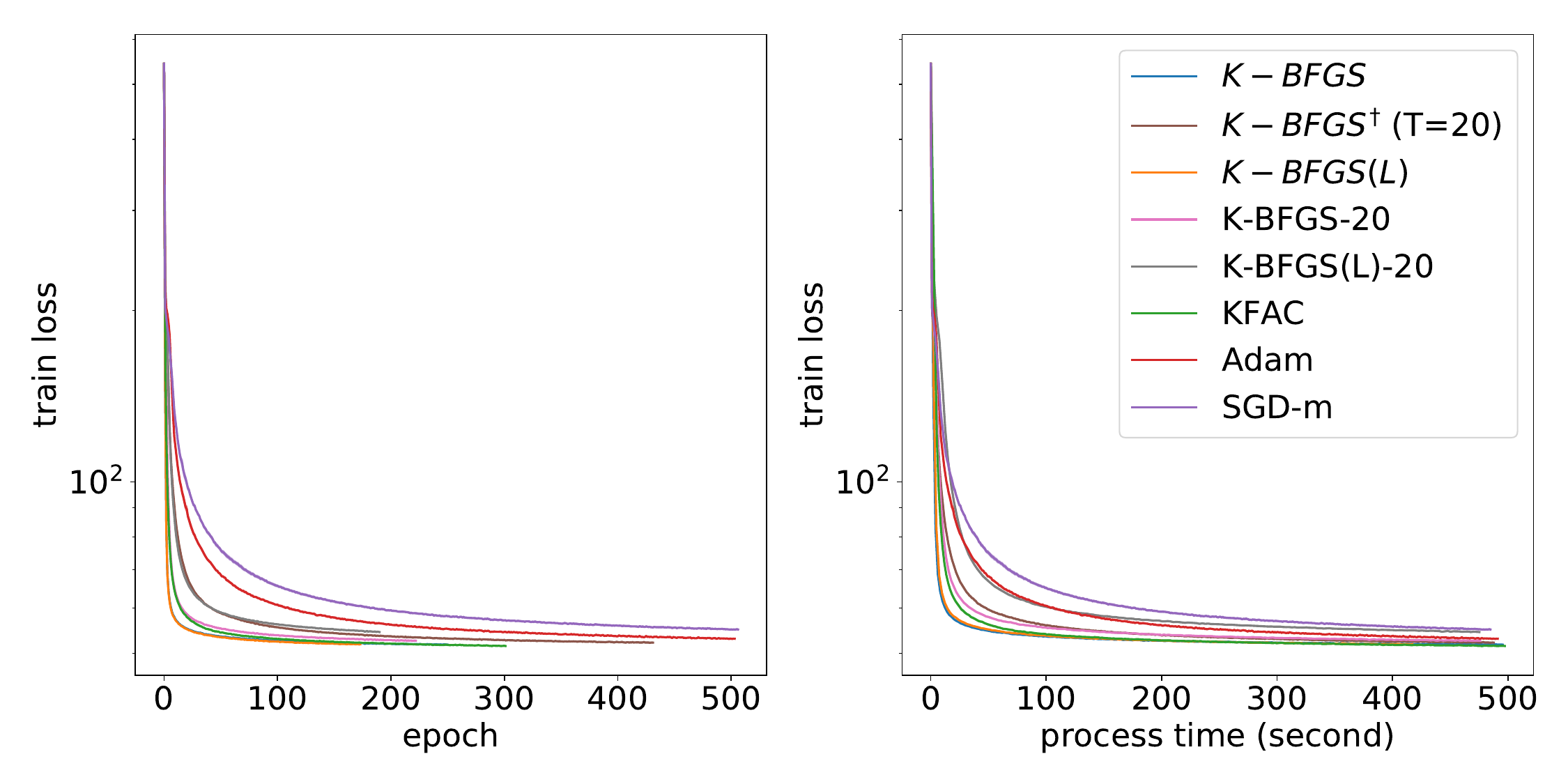}
  \caption{
   Optimization performance of K-BFGS, K-BFGS(L), their counterpart in \citet{goldfarb2020practical}, KFAC, Adam, and SGD-m on MNIST
  }
  \label{fig_3}
\end{figure}

\begin{figure}[H]
  \centering
    \includegraphics[width=0.6\textwidth]{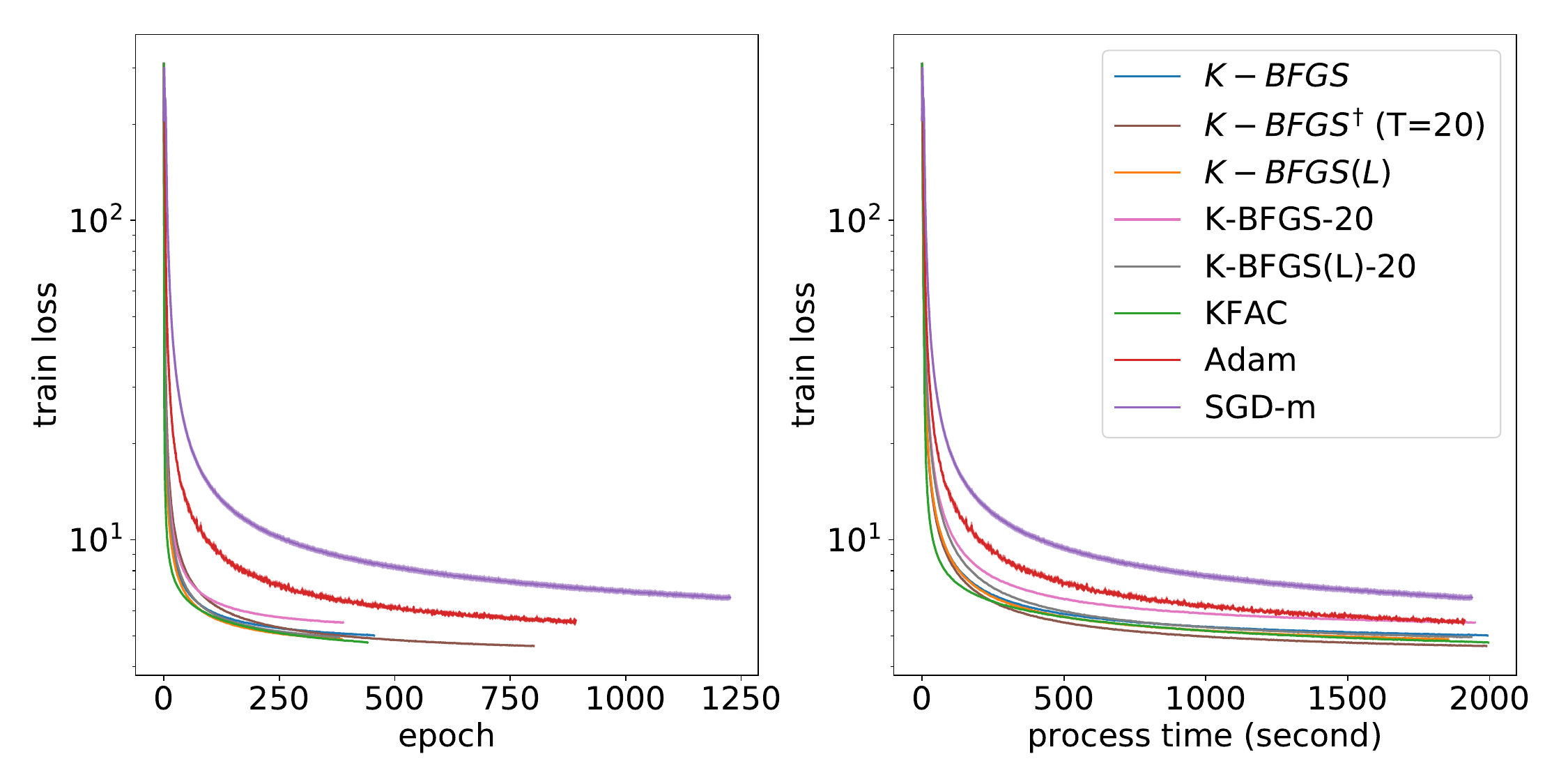}
  \caption{Optimization performance of K-BFGS, K-BFGS(L), their counterpart in \citet{goldfarb2020practical}, KFAC, Adam, and SGD-m on FACES
  }
  \label{fig_4}
\end{figure}

\begin{figure}[H]
  \centering
    \includegraphics[width=0.6\textwidth]{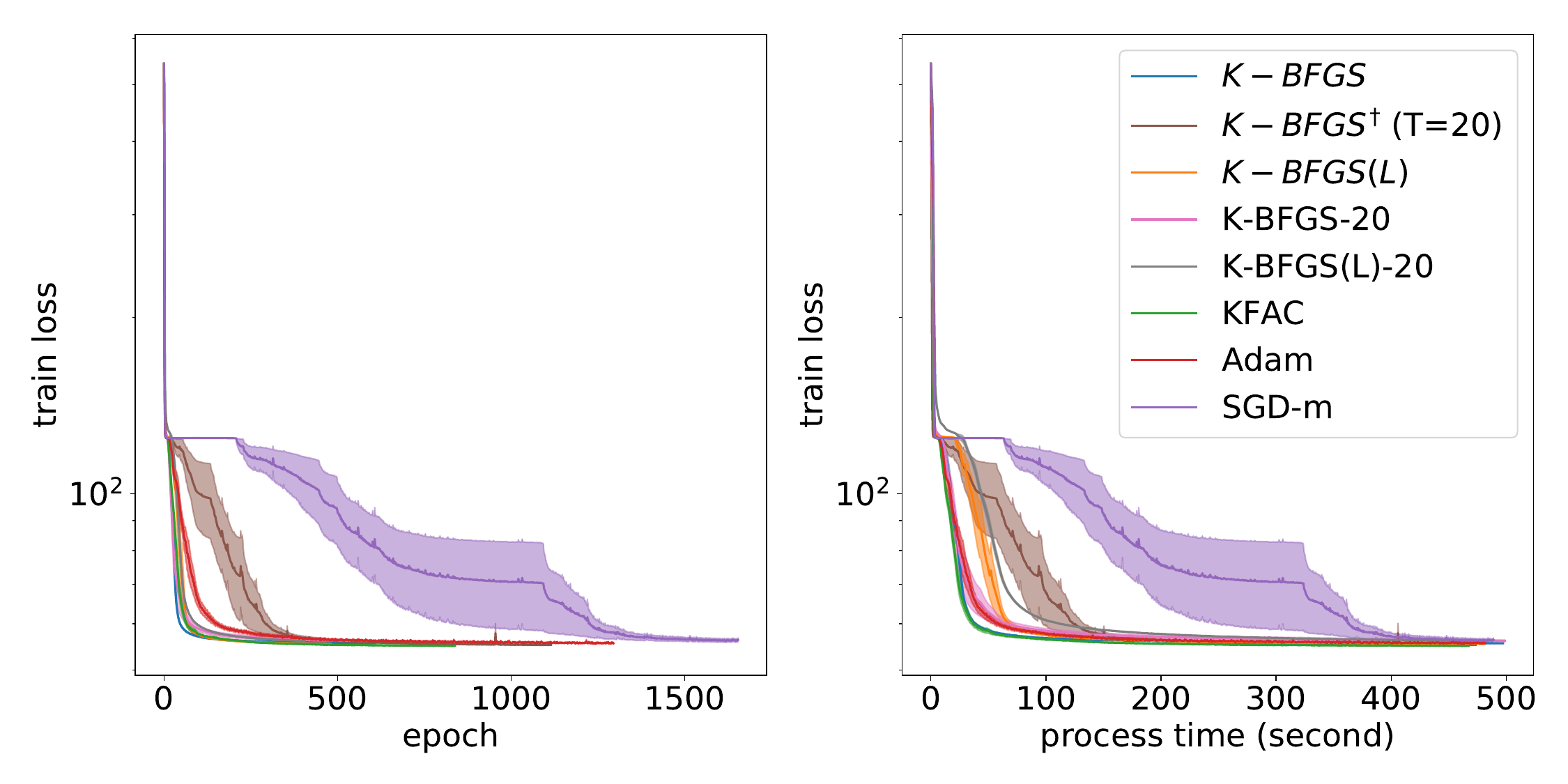}
  \caption{Optimization performance of K-BFGS, K-BFGS(L), their counterpart in \citet{goldfarb2020practical}, KFAC, Adam, and SGD-m on CURVES
  }
  \label{fig_5}
\end{figure}

\subsubsection{An Ablation Study}
\label{sec_16}

Besides the comparison presented in Section \ref{sec_17}, we also conducted an ablation study on the two generic improvements we presented in Section \ref{sec_13}. To be more specific, if both of the improvements are turned on, the algorithm is exactly the same as the one named "K-BFGS" in Table \ref{table_1}, whereas if both are turned off, it is the same as "K-BFGS-20" in Table \ref{table_1}. 

\begin{table}[H]
  \caption{
  Training loss with two improvements turned on or off. Reported values are averaged across 5 different random seeds, using the best HP values for each algorithm. 
  Improvement \#1 refers to the use of $D_{P} D_{LM}$, whereas improvement \#2 refers to the use of "minibatched" Hessian-action BFGS 
  }
  \label{table_8}
  \centering
  \begin{tabular}{llllll}
    \toprule
    Name of algorithm & Improvement \#1 & Improvement \#2 & MNIST & FACES & CURVES \\
    \midrule
    K-BFGS & 
    yes & yes & 51.60 & 5.00 & 55.46
    \\
    K-BFGS (\#1 off) &no & yes & 51.92 & 5.39 & 55.86
    \\
    K-BFGS (\#2 off) & yes & no & 51.45 & 5.26 & 55.88
    \\
    K-BFGS-20 & no & no & 52.38 & 5.46 & 56.00
    \\
    \bottomrule
  \end{tabular}
\end{table}

We repeated the same MLP autoencoder experiments described in Section \ref{sec_17}, and presented the results on four different variants in Table \ref{table_8}, which shows that using each one of the improvements alone yields better results than
the variant without improvements (i.e. K-BFGS-20),
and using the two together (i.e. K-BFGS) usually yields the best results. This ablation study, along with the reasoning in Section \ref{sec_13}, justifies the inclusion of the improvements we proposed.


\subsection{
Details on the CNN Experiments}
\label{sec_5}

The VGG16 model refers to the "model D" in \citet{simonyan2014very}, with the modifications that the 3 fully-connected (FC) layers at the end of the model being replaced with only one FC layer (input size equal the size of the output size of the last conv layer, and output size equal number of classes of the dataset), and a batch normalization layer is added after each of the convolutional layers in the model.
{These changes are usually adopted nowadays on top of the original VGG models.}
The ResNet32 model refers to the one in Table 6 of \citet{he2016deep}.



For all the algorithms that we tested, we use the weight decay technique to help improve generalization, which has shown to be effective for both 1st-order \citep{loshchilov2018decoupled} and 2nd-order methods \citep{zhang2018three}. To be more specific, take K-BFGS/K-BFGS(L) (Algorithm \ref{algo_8}) as an example, we replace Line \ref{line_14} with $W_l = W_l - \alpha_k (p_l + \gamma W_l)$
where $\gamma$ is the weigh decay factor. The same modification is done for SGD-m, Adam, and KFAC as well.

In order to obtain the results in Table \ref{table_7}, we first conducted a grid search for each algorithm based on the following ranges:

\begin{itemize}
    \item
    K-BFGS and K-BFGS(L):
    \begin{itemize}
        \item initial learning rate: $\{$ 0.03, 0.1, 0.3, 1, 3, 10, 30, 100, 300, 1e3, 3e3 $\}$
        
        \item weight decay $\gamma$: $\{$ 1e-7, 1e-6, 1e-5, 1e-4, 1e-3, 0.01, 0.1 $\}$
        
        \item damping (i.e., $\lambda$ in Algorithm \ref{algo_8}): $\{$ 1, 10, 100, 1e3, 1e4, 1e5 $\}$

        
        

    \end{itemize}

\item KFAC: 
\begin{itemize}
    \item initial learning rate: $\{$ 1e-3, 3e-3, 0.01, 0.03, 0.1, 0.3 $\}$
    
    \item weight decay $\gamma$: $\{$ 0.001, 0.01, 0.1, 1 $\}$
    
    \item damping (i.e., $\lambda$ in Algorithm \ref{algo_1}): $\{$ 1e-4, 0.001, 0.01, 0.1, 1, 10, 100 $\}$
\end{itemize}

\item Adam:
\begin{itemize}
    \item initial learning rate: $\{$ 3e-5, 1e-4, 3e-4, 1e-3, 3e-3, 0.01, 0.03, 0.1 $\}$
    
    \item weight decay $\gamma$: $\{$ 0.01, 0.1, 1, 10 $\}$
    
    \item damping (i.e., the $\epsilon$ HP in \citet{kingma2014adam}): $\{$ 1e-8, 1e-4, 0.01, 0.1, 1 $\}$
\end{itemize}

\item SGD-m: 

\begin{itemize}
    \item initial learning rate: $\{$ 3e-4, 1e-3, 3e-3, 0.01, 0.03, 0.1, 0.3 $\}$
    
    \item weight decay $\gamma$: $\{$ 1e-3, 0.01, 0.1, 1 $\}$
\end{itemize}
    
\end{itemize}

We selected the HP values that achieves the largest classification accuracy on the validation set (see Table \ref{table_6}). We then ran each algorithm with their corresponding best HP values and 5 different random seeds, and reported the average validation classification accuracy in Table \ref{table_7}. 
The training cross entropy loss (upper rows) and validation classification error (lower rows) against number of epochs (left columns) and process time (right columns) are also included in Figure \ref{fig_6} (in Section \ref{sec_9}), and Figure \ref{fig_7}, \ref{fig_8}, and \ref{fig_9} (below).


\begin{table}[tb]
  \caption{
  Best HP values (initial learning rate, weight decay, damping) for Table \ref{table_7} as well as Figures \ref{fig_6}, \ref{fig_7}, \ref{fig_8}, and \ref{fig_9}
  }
  \label{table_6}
  \vskip 0.15in
  \centering
  \begin{tabular}{ccccccc}
    \toprule
    & K-BFGS
    & K-BFGS(L)
    & KFAC
    & Adam
    & SGD-m
    \\
    \midrule
    VGG16, CIFAR10
    & (30, 1e-5, 1e4)
    & (1, 1e-3, 100)
    & (0.01, 0.1, 10)
    & (0.003, 0.1, 0.1)
    & (0.003, 0.1, -)
    \\
    ResNet32, CIFAR10
    & (100, 1e-5, 1e3)
    & (1e3, 1e-6, 1e4)
    & (0.01, 0.1, 0.01)
    & (0.003, 0.1, 0.01)
    & (0.03, 0.01, -)
    \\
    VGG16, CIFAR100
    & (0.1, 0.01, 10)
    & (0.3, 0.001, 100)
    & (0.01, 0.1, 1)
    & (3e-4, 1, 0.01)
    & (0.003, 0.1, -)
    \\
    ResNet32, CIFAR100
    & (1e3, 1e-6, 1e4)
    & (10, 1e-4, 100)
    & (0.01, 0.1, 0.001)
    & (0.01, 0.1, 0.01)
    & (0.03, 0.01, -)
    \\
    \bottomrule
  \end{tabular}
\end{table}


From Table \ref{table_6}, we can see that, the optimal damping value for K-BFGS and K-BFGS(L) tends to be larger than that for KFAC, which is somewhat reasonable since they use quasi-Newton approaches to estimate curvature information. Hence, a stronger damping term (regularization) is needed. Moreover, in our experiments, for K-BFGS and K-BFGS(L), there was a strong positive correlation between the optimal learning rate and damping values, and a strong negative correlation between the optimal learning rate and weight decay values. These are not surprising because the "effective" learning rate involves the ratio of the learning rate to the damping, and the "effective" weight decay factor is the product of the weight decay value and the learning rate. 



    



\begin{figure}[H]
  \centering
    \includegraphics[width=0.55\textwidth,height=8cm]{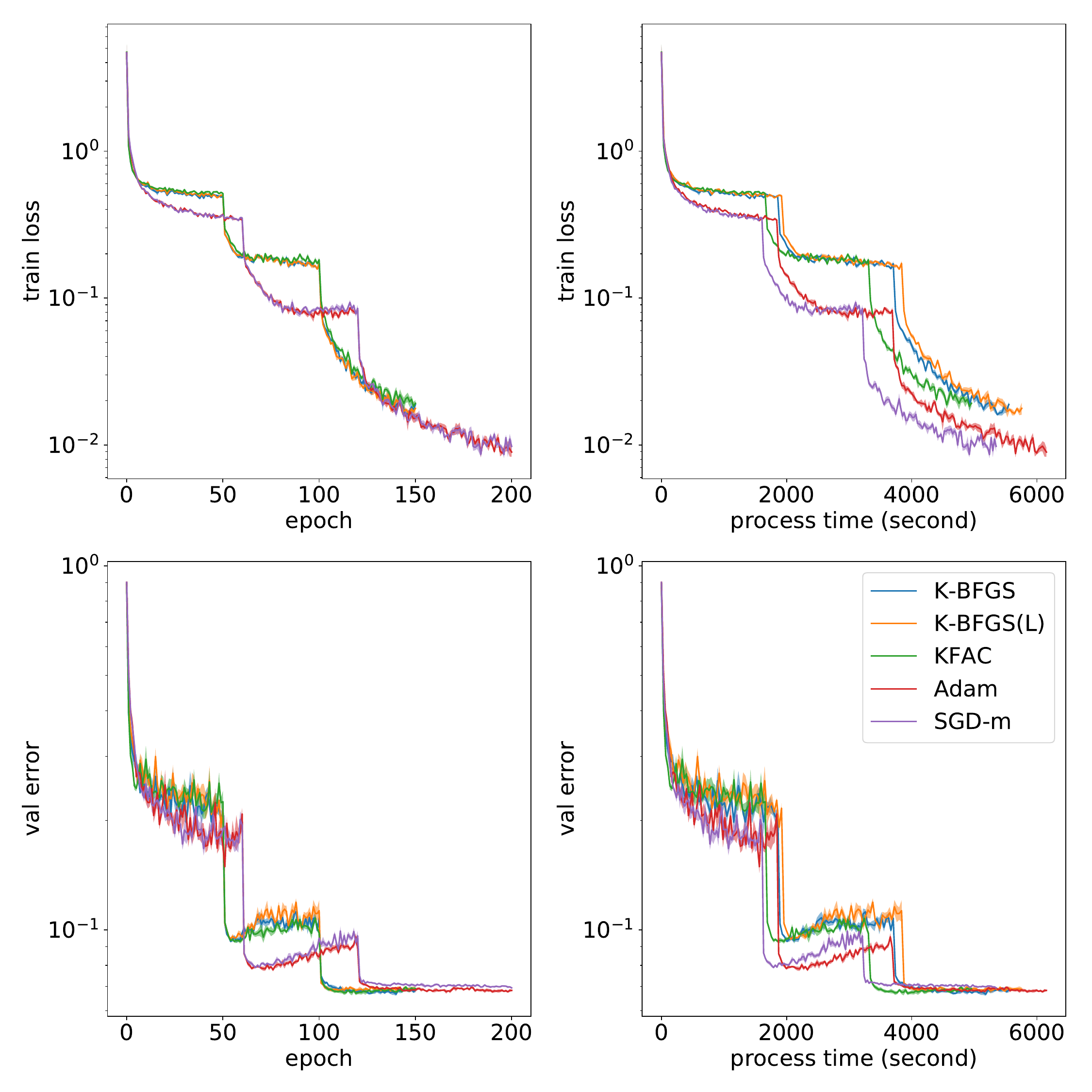}
  \caption{
   Performance of K-BFGS, K-BFGS(L), KFAC, Adam, and SGD-m on ResNet32 with CIFAR10 
  }
  \label{fig_7}
\end{figure}

\begin{figure}[H]
  \centering
    \includegraphics[width=0.6\textwidth,height=8cm]{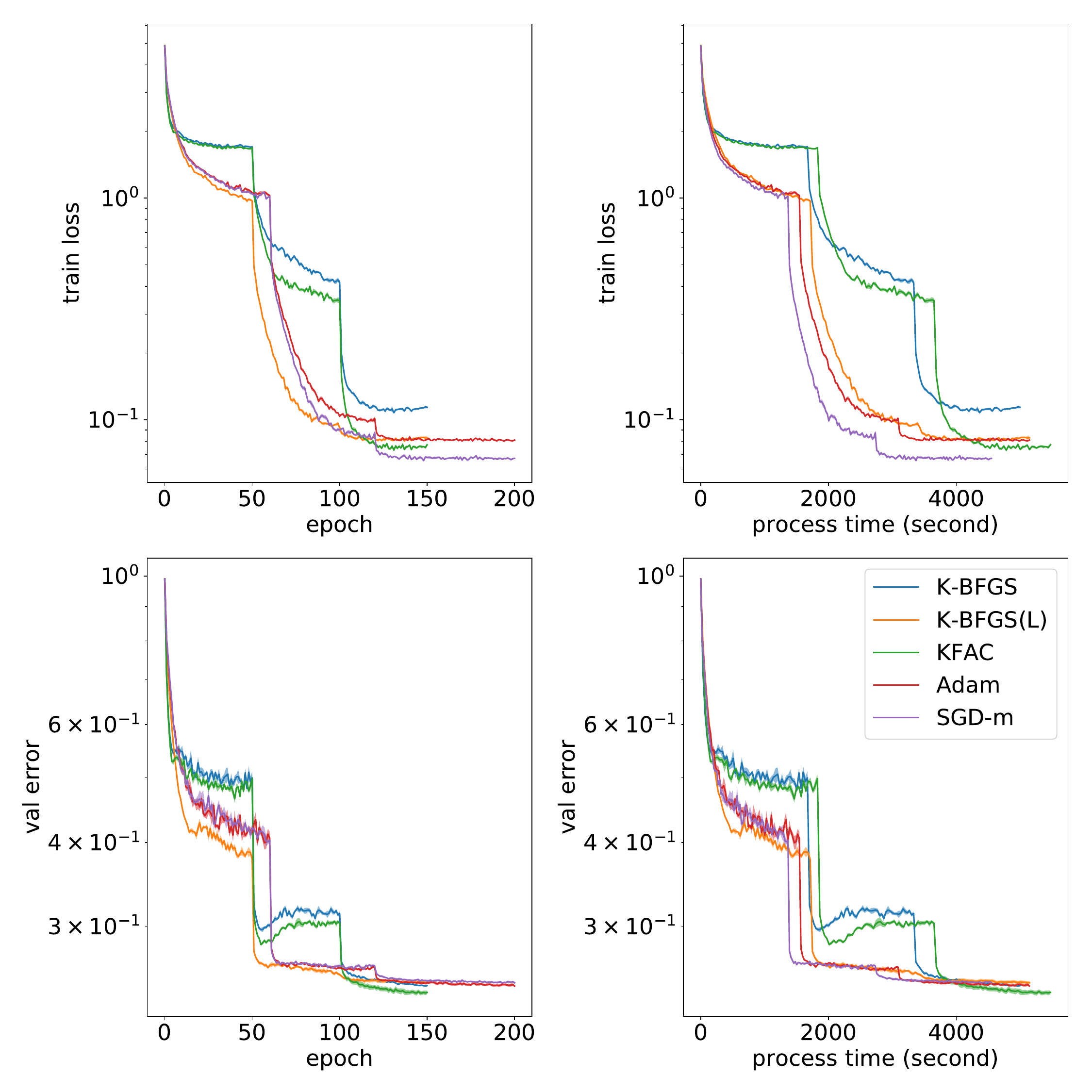}
  \caption{
   Performance of K-BFGS, K-BFGS(L), KFAC, Adam, and SGD-m on VGG16 with CIFAR100
  }
  \label{fig_8}
\end{figure}

\begin{figure}[H]
  \centering
    \includegraphics[width=0.6\textwidth,height=8cm]{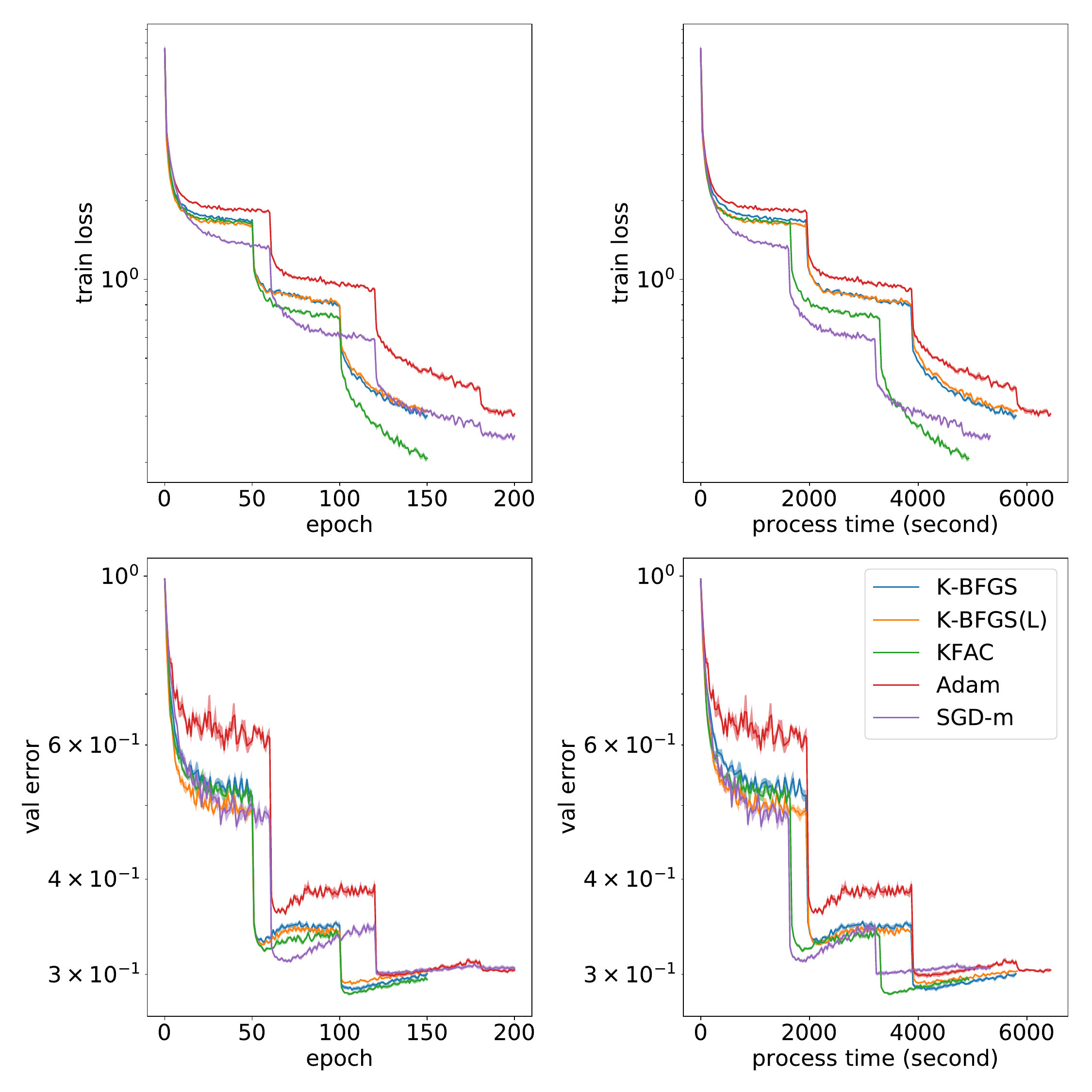}
  \caption{
   Performance of K-BFGS, K-BFGS(L), KFAC, Adam, and SGD-m on ResNet32 with CIFAR100
  }
  \label{fig_9}
\end{figure}

%% file: sections/appendix/implementation.tex
\subsection{Specification on Comparing Algorithms}
\label{sec_10}

\begin{algorithm}[ht]
    \caption{KFAC
    }
    \label{algo_1}
    \begin{algorithmic}[1]

    \REQUIRE Given learning rates $\{ \alpha_k \}$, damping value $\lambda$, batch size $m$,
    {statistics update frequency $T_1$, inverse update frequency $T_2$}
    
    \STATE $\beta = 0.9$
    
    \STATE $\widehat{\mathcal{D} W_l} = 0$, $\Omega_l = \mathbb{E}_n \left[ \sum_{t \in \mathcal{T}} \va_t^{l}(n) \va_t^{l}(n)^\top \right]$, $\Gamma_l = \mathbb{E}_n \left[ \overline{\mathcal{D} \vh^{l}(n) (\mathcal{D} \vh^{l} (n))^\top} \right]$ with $y$ sampled from the predictive distribution ($l = 1, ..., L$)
    \COMMENT{Initialization}

    \FOR {$k=1,2,\ldots$}
        \STATE Sample minibatch $M_k$ of size $m$
        
        \STATE Perform a forward-backward pass over the current minibatch $M_k$ to compute $\widetilde{\mathcal{D} W_l}$ for $l = 1, ..., L$
        
        \FOR {$l=1,2,\ldots L$}
            
            \STATE $\widehat{\mathcal{D} W_l} = \beta \widehat{\mathcal{D} W_l} + \widetilde{\mathcal{D} W_l}$
            
            \STATE
            $p_l = H_{\Gamma}^{l} \widehat{\mathcal{D} W_l} H_{\Omega}^l$
            \label{line_5}

            \STATE $W_l = W_l - \alpha_k p_{l}$.
        \ENDFOR
        
        \IF{$k \equiv 0 \pmod{T_1}$}
        
        \STATE Perform another pass over $M_k$ with $y$ sampled from the predictive distribution to compute $\mathcal{D} \vh_t^l$ for $l = 1, ..., L$
        
        \FOR {$l=1,2,\ldots L$}
        
        \STATE Update
        $\Omega_l = \beta \cdot \Omega_l + (1-\beta) \cdot \widetilde{\sum_{t \in \mathcal{T}} \va_t^{l} (\va_t^{l})^\top}$,
        $\Gamma_l = \beta \cdot \Gamma_l + (1-\beta) \cdot \widetilde{\overline{\mathcal{D} \vh_t^{l} (\mathcal{D} \vh_t^{l})^\top}}$
        \label{line_3}
        
            
        
        
        \ENDFOR
        
        \ENDIF

        \IF{$k \equiv 0 \pmod{T_2}$}
        
        \FOR {$l=1,2,\ldots L$}
            
        \STATE 
        Recompute $H_{\Omega}^l = (\Omega_l + \pi_l \sqrt{\lambda} I)^{-1}$,
        $H_{\Gamma}^l = (\Gamma_l + \frac{1}{\pi_l} \sqrt{\lambda} I)^{-1}$, where $\pi_l = \sqrt{\frac{\text{trace}(\Omega_l \otimes I)}{\text{trace}(I \otimes \Gamma_l)}}$
        \alglinelabel{line_11}
        
        \ENDFOR
        
        \ENDIF

        \ENDFOR  
    
    \end{algorithmic}
\end{algorithm}

We describe the version of KFAC that we implemented in Algorithm \ref{algo_1}. 
Note that $\Omega_l$ in KFAC is the same as $A_l$ in K-BFGS. 
Similar to the pseudo-code of K-BFGS, we assume that all layers are convolutional. However, one can easily derive our KFAC implementation for fully-connected layers from Algorithm \ref{algo_1} and \citet{martens2015optimizing}.

Note that KFC-pre in \citet{grosse2016kronecker} differs from Algorithm \ref{algo_1} in the following ways:
\begin{itemize}
    
    
    \item KFC-pre uses clipping for the approximated natural gradient direction $p$;
    
    \item KFC-pre uses momentum for $p$;
    
    \item KFC-pre
    uses parameter averaging on $\theta$.

\end{itemize}
All of these techniques can also be applied to K-BFGS. Since we are primarily interested in comparing different pre-conditioning matrices, we chose not to include such techniques in our implementation.

In our KFAC implementation, for the CNN problems that have batch normalization (BN) layers, we update the parameters of the BN layers with the gradient direction, along with the same learning rate $\alpha$, as was done in \citet{zhang2018three}.

Also note that in Algorithm \ref{algo_1}, a warm start computation of $\Omega_l$ and $\Gamma_l$ is included, i.e. initial estimates of $\Omega_l$ and $\Gamma_l$ are computed from the whole dataset before the first iteration. A similar warm start computation of $A_l$ was also included in K-BFGS. Since these warm start computations take no more than the time for one full epoch, we did not include the times for warm starts in the figures.

Finally, Adam was implemented exactly as in \citet{kingma2014adam}, with $\beta_1 = 0.9$ and $\beta_2 = 0.999$, as suggested in the paper. We view the hyper-parameter $\epsilon$ in Adam as the damping term, and tune it in our experiments (specified below). 

In SGD with momentum, the momentum of gradient is computed as in K-BFGS (Algorithm \ref{algo_8}) and KFAC (Algorithm \ref{algo_1}). In other words, if $g$ denotes the minibatch gradient, we update the momentum of graident by $\hat{g} = \beta \hat{g} + g$ with $\beta = 0.9$ at every iteration.

%% file: example_paper.bbl
\begin{thebibliography}{40}
\providecommand{\natexlab}[1]{#1}
\providecommand{\url}[1]{\texttt{#1}}
\expandafter\ifx\csname urlstyle\endcsname\relax
  \providecommand{\doi}[1]{doi: #1}\else
  \providecommand{\doi}{doi: \begingroup \urlstyle{rm}\Url}\fi

\bibitem[Amari et~al.(2000)Amari, Park, and Fukumizu]{amari2000adaptive}
Amari, S.-I., Park, H., and Fukumizu, K.
\newblock Adaptive method of realizing natural gradient learning for multilayer
  perceptrons.
\newblock \emph{Neural computation}, 12\penalty0 (6):\penalty0 1399--1409,
  2000.

\bibitem[Bakker et~al.(2018)Bakker, Henry, and Hodas]{bakker2018outer}
Bakker, C., Henry, M.~J., and Hodas, N.~O.
\newblock The outer product structure of neural network derivatives.
\newblock \emph{arXiv preprint arXiv:1810.03798}, 2018.

\bibitem[Botev et~al.(2017)Botev, Ritter, and Barber]{botev2017practical}
Botev, A., Ritter, H., and Barber, D.
\newblock Practical {G}auss-{N}ewton optimisation for deep learning.
\newblock In \emph{Proceedings of the 34th International Conference on Machine
  Learning-Volume 70}, pp.\  557--565. JMLR. org, 2017.

\bibitem[Broyden(1970)]{broyden1970convergence}
Broyden, C.~G.
\newblock The convergence of a class of double-rank minimization algorithms 1.
  general considerations.
\newblock \emph{IMA Journal of Applied Mathematics}, 6\penalty0 (1):\penalty0
  76--90, 1970.

\bibitem[Byrd et~al.(1994)Byrd, Nocedal, and Schnabel]{byrd1994representations}
Byrd, R.~H., Nocedal, J., and Schnabel, R.~B.
\newblock Representations of quasi-{N}ewton matrices and their use in limited
  memory methods.
\newblock \emph{Mathematical Programming}, 63\penalty0 (1-3):\penalty0
  129--156, 1994.

\bibitem[Byrd et~al.(2016)Byrd, Hansen, Nocedal, and
  Singer]{byrd2016stochastic}
Byrd, R.~H., Hansen, S.~L., Nocedal, J., and Singer, Y.
\newblock A stochastic quasi-{N}ewton method for large-scale optimization.
\newblock \emph{SIAM Journal on Optimization}, 26\penalty0 (2):\penalty0
  1008--1031, 2016.

\bibitem[Choi et~al.(2019)Choi, Shallue, Nado, Lee, Maddison, and
  Dahl]{choi2019empirical}
Choi, D., Shallue, C.~J., Nado, Z., Lee, J., Maddison, C.~J., and Dahl, G.~E.
\newblock On empirical comparisons of optimizers for deep learning.
\newblock \emph{arXiv preprint arXiv:1910.05446}, 2019.

\bibitem[Duchi et~al.(2011)Duchi, Hazan, and Singer]{duchi2011adaptive}
Duchi, J., Hazan, E., and Singer, Y.
\newblock Adaptive subgradient methods for online learning and stochastic
  optimization.
\newblock \emph{Journal of Machine Learning Research}, 12\penalty0
  (Jul):\penalty0 2121--2159, 2011.

\bibitem[Fletcher(1970)]{fletcher1970new}
Fletcher, R.
\newblock A new approach to variable metric algorithms.
\newblock \emph{The computer journal}, 13\penalty0 (3):\penalty0 317--322,
  1970.

\bibitem[George et~al.(2018)George, Laurent, Bouthillier, Ballas, and
  Vincent]{george2018fast}
George, T., Laurent, C., Bouthillier, X., Ballas, N., and Vincent, P.
\newblock Fast approximate natural gradient descent in a {K}ronecker factored
  eigenbasis.
\newblock In \emph{Advances in Neural Information Processing Systems}, pp.\
  9550--9560, 2018.

\bibitem[Goldfarb(1970)]{goldfarb1970family}
Goldfarb, D.
\newblock A family of variable-metric methods derived by variational means.
\newblock \emph{Mathematics of computation}, 24\penalty0 (109):\penalty0
  23--26, 1970.

\bibitem[Goldfarb et~al.(2020)Goldfarb, Ren, and
  Bahamou]{goldfarb2020practical}
Goldfarb, D., Ren, Y., and Bahamou, A.
\newblock Practical quasi-{N}ewton methods for training deep neural networks.
\newblock In Larochelle, H., Ranzato, M., Hadsell, R., Balcan, M.~F., and Lin,
  H. (eds.), \emph{Advances in Neural Information Processing Systems},
  volume~33, pp.\  2386--2396. Curran Associates, Inc., 2020.
\newblock URL
  \url{https://proceedings.neurips.cc/paper/2020/file/192fc044e74dffea144f9ac5dc9f3395-Paper.pdf}.

\bibitem[Gower et~al.(2016)Gower, Goldfarb, and
  Richt{\'a}rik]{gower2016stochastic}
Gower, R., Goldfarb, D., and Richt{\'a}rik, P.
\newblock Stochastic block {BFGS}: Squeezing more curvature out of data.
\newblock In \emph{International Conference on Machine Learning}, pp.\
  1869--1878, 2016.

\bibitem[Grosse \& Martens(2016)Grosse and Martens]{grosse2016kronecker}
Grosse, R. and Martens, J.
\newblock A {K}ronecker-factored approximate fisher matrix for convolution
  layers.
\newblock In \emph{International Conference on Machine Learning}, pp.\
  573--582, 2016.

\bibitem[Gupta et~al.(2018)Gupta, Koren, and Singer]{gupta2018shampoo}
Gupta, V., Koren, T., and Singer, Y.
\newblock Shampoo: Preconditioned stochastic tensor optimization.
\newblock In Dy, J. and Krause, A. (eds.), \emph{Proceedings of the 35th
  International Conference on Machine Learning}, volume~80 of \emph{Proceedings
  of Machine Learning Research}, pp.\  1842--1850. PMLR, 2018.

\bibitem[He et~al.(2016)He, Zhang, Ren, and Sun]{he2016deep}
He, K., Zhang, X., Ren, S., and Sun, J.
\newblock Deep residual learning for image recognition.
\newblock In \emph{Proceedings of the IEEE conference on computer vision and
  pattern recognition}, pp.\  770--778, 2016.

\bibitem[Heskes(2000)]{heskes2000}
Heskes, T.
\newblock On "natural" learning and pruning in multilayered perceptrons.
\newblock \emph{Neural Computation}, 12, 01 2000.
\newblock \doi{10.1162/089976600300015637}.

\bibitem[Hinton et~al.(2012)Hinton, Srivastava, and Swersky]{hinton2012neural}
Hinton, G., Srivastava, N., and Swersky, K.
\newblock Neural networks for machine learning lecture 6a overview of
  mini-batch gradient descent.
\newblock \emph{Cited on}, 14\penalty0 (8), 2012.

\bibitem[Hinton \& Salakhutdinov(2006)Hinton and
  Salakhutdinov]{hinton2006reducing}
Hinton, G.~E. and Salakhutdinov, R.~R.
\newblock Reducing the dimensionality of data with neural networks.
\newblock \emph{science}, 313\penalty0 (5786):\penalty0 504--507, 2006.

\bibitem[Kingma \& Ba(2014)Kingma and Ba]{kingma2014adam}
Kingma, D. and Ba, J.
\newblock Adam: A method for stochastic optimization.
\newblock \emph{International Conference on Learning Representations}, 2014.

\bibitem[Krizhevsky et~al.(2009)Krizhevsky, Hinton,
  et~al.]{krizhevsky2009learning}
Krizhevsky, A., Hinton, G., et~al.
\newblock Learning multiple layers of features from tiny images.
\newblock 2009.

\bibitem[Krizhevsky et~al.(2012)Krizhevsky, Sutskever, and
  Hinton]{NIPS2012_c399862d}
Krizhevsky, A., Sutskever, I., and Hinton, G.~E.
\newblock Imagenet classification with deep convolutional neural networks.
\newblock In Pereira, F., Burges, C. J.~C., Bottou, L., and Weinberger, K.~Q.
  (eds.), \emph{Advances in Neural Information Processing Systems}, volume~25,
  pp.\  1097--1105. Curran Associates, Inc., 2012.
\newblock URL
  \url{https://proceedings.neurips.cc/paper/2012/file/c399862d3b9d6b76c8436e924a68c45b-Paper.pdf}.

\bibitem[LeCun et~al.(1998)LeCun, Bottou, Bengio, and
  Haffner]{lecun1998gradient}
LeCun, Y., Bottou, L., Bengio, Y., and Haffner, P.
\newblock Gradient-based learning applied to document recognition.
\newblock \emph{Proceedings of the IEEE}, 86\penalty0 (11):\penalty0
  2278--2324, 1998.

\bibitem[Liu \& Nocedal(1989)Liu and Nocedal]{liu1989limited}
Liu, D.~C. and Nocedal, J.
\newblock On the limited memory {BFGS} method for large scale optimization.
\newblock \emph{Mathematical programming}, 45\penalty0 (1-3):\penalty0
  503--528, 1989.

\bibitem[Loshchilov \& Hutter(2019)Loshchilov and
  Hutter]{loshchilov2018decoupled}
Loshchilov, I. and Hutter, F.
\newblock Decoupled weight decay regularization.
\newblock In \emph{International Conference on Learning Representations}, 2019.
\newblock URL \url{https://openreview.net/forum?id=Bkg6RiCqY7}.

\bibitem[Martens(2010)]{martens2010deep}
Martens, J.
\newblock Deep learning via hessian-free optimization.
\newblock In \emph{ICML}, volume~27, pp.\  735--742, 2010.

\bibitem[Martens \& Grosse(2015)Martens and Grosse]{martens2015optimizing}
Martens, J. and Grosse, R.
\newblock Optimizing neural networks with {K}ronecker-factored approximate
  curvature.
\newblock In \emph{International conference on machine learning}, pp.\
  2408--2417, 2015.

\bibitem[Martens et~al.(2018)Martens, Ba, and
  Johnson]{martens2018kroneckerfactored}
Martens, J., Ba, J., and Johnson, M.
\newblock Kronecker-factored curvature approximations for recurrent neural
  networks.
\newblock In \emph{International Conference on Learning Representations}, 2018.
\newblock URL \url{https://openreview.net/forum?id=HyMTkQZAb}.

\bibitem[Povey et~al.(2014)Povey, Zhang, and Khudanpur]{povey2014parallel}
Povey, D., Zhang, X., and Khudanpur, S.
\newblock Parallel training of dnns with natural gradient and parameter
  averaging.
\newblock \emph{arXiv preprint arXiv:1410.7455}, 2014.

\bibitem[Ren \& Goldfarb(2019)Ren and Goldfarb]{ren2019efficient}
Ren, Y. and Goldfarb, D.
\newblock Efficient subsampled {G}auss-{N}ewton and natural gradient methods
  for training neural networks.
\newblock \emph{arXiv preprint arXiv:1906.02353}, 2019.

\bibitem[Ren \& Goldfarb(2021)Ren and Goldfarb]{ren2021tensor}
Ren, Y. and Goldfarb, D.
\newblock Tensor normal training for deep learning models.
\newblock In Beygelzimer, A., Dauphin, Y., Liang, P., and Vaughan, J.~W.
  (eds.), \emph{Advances in Neural Information Processing Systems}, 2021.
\newblock URL \url{https://openreview.net/forum?id=-t9LPHRYKmi}.

\bibitem[Robbins \& Monro(1951)Robbins and Monro]{robbins1951stochastic}
Robbins, H. and Monro, S.
\newblock A stochastic approximation method.
\newblock \emph{The annals of mathematical statistics}, pp.\  400--407, 1951.

\bibitem[Shanno(1970)]{shanno1970conditioning}
Shanno, D.~F.
\newblock Conditioning of quasi-{N}ewton methods for function minimization.
\newblock \emph{Mathematics of computation}, 24\penalty0 (111):\penalty0
  647--656, 1970.

\bibitem[Simonyan \& Zisserman(2014)Simonyan and Zisserman]{simonyan2014very}
Simonyan, K. and Zisserman, A.
\newblock Very deep convolutional networks for large-scale image recognition.
\newblock \emph{arXiv preprint arXiv:1409.1556}, 2014.

\bibitem[Vinyals \& Povey(2012)Vinyals and Povey]{vinyals2012krylov}
Vinyals, O. and Povey, D.
\newblock Krylov subspace descent for deep learning.
\newblock In \emph{Artificial Intelligence and Statistics}, pp.\  1261--1268,
  2012.

\bibitem[Wang et~al.(2017)Wang, Ma, Goldfarb, and Liu]{wang2017stochastic}
Wang, X., Ma, S., Goldfarb, D., and Liu, W.
\newblock Stochastic quasi-{N}ewton methods for nonconvex stochastic
  optimization.
\newblock \emph{SIAM Journal on Optimization}, 27\penalty0 (2):\penalty0
  927--956, 2017.

\bibitem[Wu et~al.(2017)Wu, Mansimov, Grosse, Liao, and Ba]{wu2017scalable}
Wu, Y., Mansimov, E., Grosse, R.~B., Liao, S., and Ba, J.
\newblock Scalable trust-region method for deep reinforcement learning using
  {K}ronecker-factored approximation.
\newblock \emph{Advances in neural information processing systems},
  30:\penalty0 5279--5288, 2017.

\bibitem[Wu et~al.(2020)Wu, Zhu, Wu, Wang, and Ge]{wu2020dissecting}
Wu, Y., Zhu, X., Wu, C., Wang, A., and Ge, R.
\newblock Dissecting hessian: Understanding common structure of hessian in
  neural networks.
\newblock \emph{arXiv preprint arXiv:2010.04261}, 2020.

\bibitem[Xu et~al.(2019)Xu, Roosta, and Mahoney]{xu2019newton}
Xu, P., Roosta, F., and Mahoney, M.~W.
\newblock Newton-type methods for non-convex optimization under inexact hessian
  information.
\newblock \emph{Mathematical Programming}, pp.\  1--36, 2019.

\bibitem[Zhang et~al.(2019)Zhang, Wang, Xu, and Grosse]{zhang2018three}
Zhang, G., Wang, C., Xu, B., and Grosse, R.
\newblock Three mechanisms of weight decay regularization.
\newblock In \emph{International Conference on Learning Representations}, 2019.
\newblock URL \url{https://openreview.net/forum?id=B1lz-3Rct7}.

\end{thebibliography}
